\documentclass[10pt,twocolumn]{article} 
\usepackage[T1]{fontenc}
\usepackage{lmodern}
\usepackage{amsmath, amssymb, amsfonts, amsthm}
\usepackage{scalefnt}
\usepackage{mathtools}
\usepackage{bm}
\usepackage[numbers]{natbib} 

\usepackage[utf8]{inputenc}
\usepackage{geometry}
\newcommand\blfootnote[1]{%
  \begingroup
  \renewcommand\thefootnote{}\footnote{#1}%
  \addtocounter{footnote}{-1}%
  \endgroup
}

\usepackage{array}           % Advanced table formatting
\usepackage{booktabs}        % Professional tables (toprule, midrule, bottomrule)
\usepackage{multirow}        % Multi-row cells
\usepackage{tabularx}        % Auto-width tables
\usepackage{threeparttable}  % Tables with footnotes

\usepackage{booktabs}       % 三线表
\usepackage{multirow}       % 跨行表格
\usepackage{graphicx}       % 图片支持
\usepackage{subcaption}     % 子图支持
\usepackage{xcolor}         % 颜色支持
\usepackage[T1]{fontenc}
\usepackage{textcomp}
\usepackage{lmodern}
\usepackage{enumitem}
\setlist{leftmargin=*, topsep=0.5em, parsep=0pt, itemsep=1em, labelindent=0pt, align=left}

\usepackage{algorithm}       % Algorithm environment
\usepackage{algorithmic}     % Algorithmic steps
\usepackage{listings}        % Code listings
\usepackage{color}           % Colors for code

\usepackage{hyperref}
\hypersetup{
    colorlinks=true,
    linkcolor=blue!80!black,
    citecolor=blue!80!black,
    urlcolor=blue!80!black
}
\usepackage[capitalize, nameinlink]{cleveref}

\theoremstyle{plain}
\newtheorem{theorem}{Theorem}[section]
\newtheorem{lemma}[theorem]{Lemma}

\newtheorem{corollary}[theorem]{Corollary}

\theoremstyle{definition}

\newtheorem{assumption}{Assumption}

\theoremstyle{remark}
\newtheorem{remark}{Remark}[section]

\usepackage{lmodern}

\title{GCUL: Ambiguity Identification in Text Emotion Classification via Cluster-Guided Learning}

\author{
    \textbf{Zhongqi FAN}$^{*}$ \quad
    \textbf{Tianyou ZHANG}$^{\dagger}$ \quad
    \textbf{Fei CHEN}$^{\dagger}$ \\
    Beijing Normal-Hong Kong Baptist University \\
    \texttt{\{u430033011, t330033051, t330033001\}@mail.bnbu.edu.cn}
}

\date{\today}

\begin{document}

\maketitle

\blfootnote{$^{\dagger}$These authors contributed equally to this work and should be considered co-second authors.}
% ==============================================================================
% Abstract
% ==============================================================================
\begin{abstract}
Selective classification enables a model to abstain from predictions on
uncertain instances, but existing approaches typically reject them
through confidence scores, predefined coverage constraints or
instance-level distance measures. These approaches may overlook the collective
geometric structure of difficult samples in learned representation spaces.

We propose \textit{Guided Clustering-based Uncertain Learning} (GCUL), a
geometric-guided selective classification framework that identifies
misclassified and ambiguous instances as a potential confusion attractor in
the representation space. GCUL uses a three-phase procedure to initialize,
cluster, and explicitly relabel this uncertain region, allowing the rejection
boundary to emerge from the underlying representation geometry rather than
from a prescribed rejection rate. We further derive a selectivity score and
a geometric sufficient condition that characterizes when rejection can
provide positive operational utility, enabling pre-deployment feasibility
assessment.

 GCUL improves DistilBERT accuracy from 89.37\% to 94.98\% with <9\% rejection. Beyond accuracy, our selectivity score correctly pre-detects the only dataset (GoEmotion) where all baselines fail, and controlled simulations yield 6.1\% Type-I and 0\% Type-II errors, validating the sufficient condition's conservatism. These results suggest that collective representation geometry provides a useful alternative perspective for selective prediction.
\end{abstract}
% ==============================================================================
% Section 1: Introduction
% ==============================================================================
\section{Introduction}
\label{sec:introduction}

Modern neural classifiers can achieve high predictive accuracy while remaining
unreliable on ambiguous or difficult inputs. In many practical applications,
an incorrect prediction can be substantially more costly than deferring an
uncertain case for further review. This motivates \emph{selective
classification}, where a classifier is allowed to abstain from predictions
that are considered unreliable.

Selective classification, or classification with a reject option, originates from
the decision-theoretic formulation of Chow~\cite{chow1970optimum}, where a classifier
may abstain from predictions deemed unreliable. Modern formulations characterize
this task through the trade-off between coverage and selective risk
\cite{elyaniv2010foundations}. 

A common approach constructs the selection function by thresholding post-hoc
confidence scores. Hendrycks and Gimpel~\cite{hendrycks2017baseline} established
Maximum Softmax Probability (MSP) as a simple baseline, while subsequent work
improved confidence estimation through calibration~\cite{guo2017calibration}
and input perturbation~\cite{liang2018enhancing}. Alternatively, learned
selective prediction methods jointly optimize the predictor and selection
function. SelectiveNet~\cite{geifman2019selectivenet}, for example, incorporates
a dedicated selection head with a prescribed coverage constraint, while Deep
Gamblers~\cite{ziyin2019deepgamblers} formulates abstention through portfolio
optimization. Despite their effectiveness, these approaches predominantly rely
on scalar confidence estimates or predefined selection constraints, which may
limit their adaptability to varying ambiguity levels.

An alternative paradigm derives uncertainty from the geometry of intermediate
representation spaces. Lee et al.~\cite{lee2018simple} employ Mahalanobis
distance under class-conditional Gaussian assumptions, while the Trust
Score~\cite{jiang2018trustscore} and deep $k$-Nearest Neighbors~\cite{sun2022knn}
exploit relative class geometry and local neighborhood structure, respectively.
These methods establish the utility of representation geometry for uncertainty
estimation, but primarily evaluate instances individually with respect to class
distributions or local neighborhoods. They do not explicitly exploit collective
geometric structures formed by hard-to-classify samples.

Hard-example mining provides complementary evidence that difficult instances
constitute an informative subset of the data distribution. OHEM
\cite{shrivastava2016ohem} and Focal Loss~\cite{lin2017focalloss} identify and
re-weight difficult samples, while hard-negative mining plays a central role in
contrastive representation learning~\cite{robinson2021hardnegative}. In NLP,
Dataset Cartography~\cite{swayamdipta2020cartography} further demonstrated that
training dynamics can partition samples into functionally distinct groups,
including easy-to-learn, hard-to-learn, and ambiguous instances.

However, existing approaches primarily characterize difficult instances through
instance-level statistics such as loss, confidence, or prediction variability.
GCUL instead investigates whether these instances exhibit collective geometric
structure in the latent space, and exploits such structure to construct a
data-adaptive rejection region.

In this paper, we study short-text multi-class emotion classification and evaluate the proposed framework on three benchmark datasets covering different numbers of emotion categories. For detailed information about the dataset, please refer to the appendix~\ref{app:dataset_stats}.

To understand why conventional selective classification fails in complex Natural Language Processing (NLP) tasks, we begin by investigating the geometric topology of error-prone representations in the embedding space of Pre-trained Language Models (PLMs).

\paragraph{Pilot observation: geometric confusion attractors}
\label{subsec:pilot_study}
Traditional selective classification paradigms predominantly rely on instance-level confidence scoring (e.g., \cite{chow1970optimum, geifman2017selective}), effectively treating prediction uncertainty as isolated, point-wise evaluation tasks. Consequently, they inherently overlook the rich geometric topology and localized manifold clustering of hard-to-classify representations in deep embedding spaces.
Consequently, existing approaches rely almost exclusively on instance-level confidence thresholds (e.g., Softmax entropy or MC-Dropout variances) to reject unreliable predictions.

\begin{figure}[H]
  \centering
  \includegraphics[width=\columnwidth]{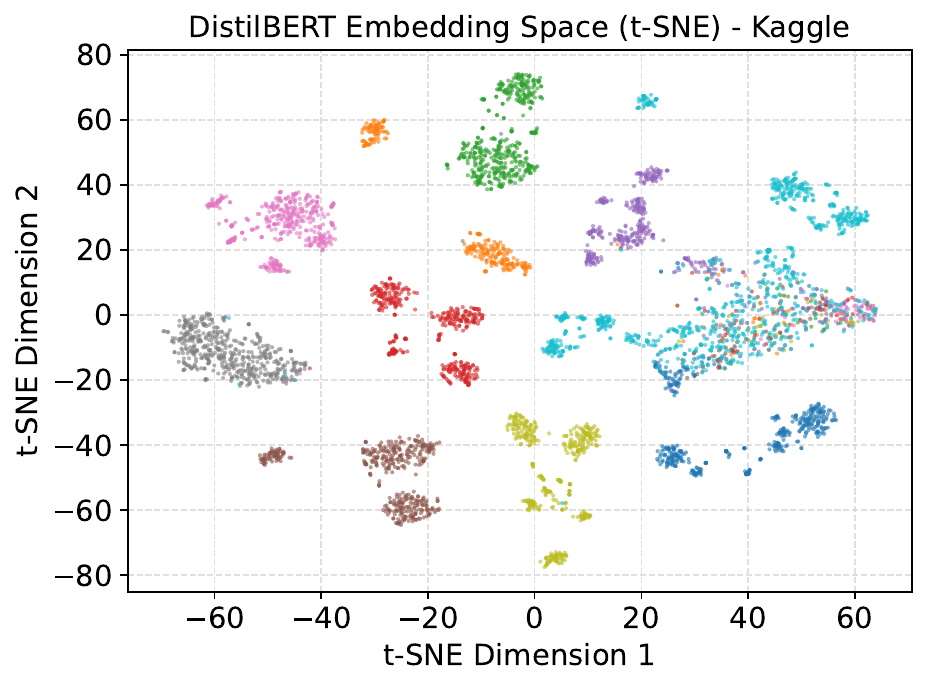}
  \caption{T-SNE (Each color represents a category.)}
  \label{fig:tsne}
\end{figure}
However, as illustrated in Figure \ref{fig:tsne}, analyzing the BERT\cite{devlin2019bert} feature representations via t-SNE visualization reveals a fundamentally different structural pattern. Rather than being uniformly dispersed as random noise, misclassified and ambiguous samples naturally form localized, compact clusters within the manifold. This phenomenon may stems from the deep contextual encoding of PLMs: inputs sharing subtle semantic ambiguities (e.g., overlapping emotion categories or domain-specific polysemy) are projected into adjacent regions in the representation space.

Through embedding space visualization, we discover that misclassified samples form a \textbf{compact, linearly separable cluster} — a phenomenon we term the ``confusion cluster''. This suggests that for strong modern encoders, a significant portion of systematic errors concentrates along intrinsic boundaries of semantic ambiguity in the representation space, rather than pure stochastic noise.

Based on this observation, we propose a \textbf{Guided Clustering-based Uncertain Learning (GCUL)} framework that explicitly models ambiguous samples by introducing an uncertainty class. GCUL operates in three phases:
    \begin{itemize}
        \item \textbf{Phase 1:} Train an initial 11-class classifier to obtain preliminary predictions and identify misclassified samples.
        \item \textbf{Phase 2:} Apply Linear Discriminant Analysis (LDA) for supervised dimensionality reduction, then perform K-Means clustering with 12 initial centers (11 class centers + 1 confusion center) to detect confusing samples. Clusters with lowest purity are relabeled as \texttt{uncertain}.
        \item \textbf{Phase 3:} Retrain a new 12-class classifier (11 original emotions + \texttt{uncertain}), allowing the model to reject ambiguous cases and focus on clearly distinguishable emotions.
    \end{itemize}

Beyond the algorithmic framework, we develop a theoretical analysis of when
such geometric rejection can provide positive operational utility. We
formulate the trade-off between error reduction and rejection cost through a
GCUL selectivity score, and derive sufficient geometric conditions in terms of
error capture, clean-sample contamination, centroid separation, and directional
dispersion. This leads to a critical rejection-cost threshold
$\lambda^*$ that can be evaluated before completing an extensive selective
classification search. The resulting analysis provides not only a mechanism
for constructing a rejection region, but also a principled criterion for
determining whether the underlying representation geometry is favorable to
selective prediction.

Empirically, benchmark experiments indicate that GCUL substantially improves selective accuracy on datasets exhibiting a concentrated confusion structure. Additionally, Comparsion experiments cross representation dimension show that GCUL achieves competitive performance without requiring manually specified rejection rates, as its rejection boundary is determined by the underlying representation geometry, whereas traditional methods remain highly sensitive to predefined operational parameters. Synthetic experiments further proved the accuracy and stability of the theoretical feasible framework.  These results suggest that the main value of GCUL is not universal superiority over existing selective classifiers, but but a data-adaptive geometric formulation of rejection accompanied by an explicit theoretical feasibility criterion.

% ==============================================================================
% Section 2: The GCUL Framework
% ==============================================================================
\section{The GCUL framework}
\label{sec:framework}

Motivated by the insight that hard-to-classify samples inhabit compact manifolds rather than isolated points, we shift the selective classification paradigm from instance-level thresholding to cluster-guided uncertainty learning.

\subsection{The three-phase GCUL pipeline}
\label{subsec:pipeline}
Based on the observation above, we propose a three-phase framework that explicitly identifies and handles confusing samples by introducing an \texttt{uncertain} class. The main objective is to separate the gathering place of confusion samples, with output uncertain, make other classifications more easily to separate. To overcome the problem that the \texttt{uncertain} samples have no special labels, unsupervised clustering was introduced. 
Algorithm~\ref{alg:gcul} summarizes the complete GCUL framework.

\begin{algorithm}[htbp]
\caption{Guided clustering-based uncertain learning (GCUL)}
\label{alg:gcul}
\begin{algorithmic}[1]
\REQUIRE Training data $\mathcal{D}_{\mathrm{train}}$, number of clusters $K = |\mathcal{Y}| + 1$, purity threshold $\tau$
\ENSURE Trained K-class classifier $f_2$
\STATE \textbf{Phase 1:} Train $f_1$ on $\mathcal{D}_{\mathrm{train}}$ ($|\mathcal{Y}|$ classes)
\STATE \textbf{Phase 2:}
\STATE \quad Reduce dimensionality: $Z_{\mathrm{red}} = \text{PCA} (Z, y)$ (Optional)
\STATE \quad Compute class medians $\mu_c$ for $c \in \mathcal{Y}$ and confusion median $\mu_{\mathrm{conf}}$ (the center of all misclassified samples in $\mathcal{D}_{\mathrm{train}}$)
\STATE \quad Run K-Means with initial centers $\{\mu_c\}_{c=1}^{|\mathcal{Y}|} \cup \{\mu_{\mathrm{conf}}\}$
\STATE \quad Calculate the purity for each cluster
\STATE \quad Label all samples of the cluster which has lowest purity.
\STATE \textbf{Phase 3:} Train $f_2$ on relabeled $\mathcal{D}_{\mathrm{train}}'$  ($|\mathcal{Y}|$+1 classes)
\RETURN $f_2$
\end{algorithmic}
\end{algorithm}

\subsection{Operational risk formulation and utility gain}
\label{subsec:risk_formulation}
While the design of GCUL originates from intuitive geometric observations—specifically, that hard-to-classify representations form compact clusters in PLM embedding spaces—relying solely on empirical intuition presents a significant challenge in selective classification. A model that achieves selective gains on one dataset may fail unpredictably on another if the underlying conditions for valid rejection are not mathematically formalized.

To transcend a purely empirical trial-and-error design, it is imperative to establish a formal theoretical boundary that governs when and why GCUL is \textit{truly effective}. Rather than assuming universal superiority, we seek to identify the precise parametric conditions under which cluster-guided uncertainty learning strictly yields positive utility over unrejected base models.

Specifically, constructing a \textit{theoretical feasibility domain} serves three core purposes:
\begin{itemize}
    \item \textbf{Eliminating empirical artifacts:} It separates genuine geometric separability from empirical hyperparameter overfitting and variance induced by data partitioning.
    \item \textbf{Defining operational boundaries:} It explicitly maps out the decision region—parameterized by rejection cost penalties and cluster error purity—where selective rejection guarantees a strict error reduction.
    \item \textbf{Providing deployable guarantees:} It offers a verifiable sufficient condition for practitioners to determine \textit{a priori} whether applying GCUL to a given task domain will produce guaranteed performance gains.
\end{itemize}

% ==============================================================================
% Section 3: Theoretical Feasibility Analysis
% ==============================================================================
\section{Theoretical feasibility analysis}
\label{sec:theoretical_analysis}
In this section, we establish the mathematical modeling for GCUL and derive the sufficient conditions that guarantee its theoretical feasibility and operational efficacy.

\subsection{Mathematical framework and foundational assumptions}
\label{subsec:assumptions1}
\paragraph{Problem formulation and mathematical notations}
Let $\mathcal{X}$ denote the input text space and $\mathcal{Y} = \{1, \dots, C\}$ represent $C$ clean categories. Given a dataset $\mathcal{D} = \{(x_i, y_i)\}_{i=1}^N$, GCUL operates via three sequential phases:

\textbf{Phase 1 (Initial Classification):} A base classifier $f_1: \mathcal{X} \to \mathcal{Y}$ (e.g., Logistic Regression) yields predictions $\hat{y}_i = f_1(x_i)$. The misclassified error pool is identified as:
\begin{equation}
\mathcal{E} \coloneqq \left\{ i \in [N] : \hat{y}_i \neq y_i \right\}, \quad \text{with base error } \varepsilon = \frac{|\mathcal{E}|}{N}.
\end{equation}

\textbf{Phase 2 (Guided Clustering and Relabeling):} Contextual embeddings $\mathbf{h}_i = \phi(x_i) \in \mathbb{R}^{d_0}$ ($d_0 = 768$) are extracted and projected via LDA matrix $\mathbf{W}_{\mathrm{LDA}} \in \mathbb{R}^{d_0 \times d}$ into reduced vectors $\mathbf{z}_i \in \mathbb{R}^d$ ($d \le C-1$). We construct $K = C+1$ initial centroids $\{\boldsymbol{\mu}_k^{(0)}\}_{k=1}^{C+1}$: $C$ clean class medians $\boldsymbol{\mu}_c^{(0)} = \mathrm{median}(\{\mathbf{z}_i : \hat{y}_i = c\})$ and one explicit \emph{confusion attractor centroid} $\boldsymbol{\mu}_{\mathrm{conf}}^{(0)} = \mathrm{median}(\{\mathbf{z}_i : i \in \mathcal{E}\})$. After $K$-Means partitions the space into disjoint clusters $\{\mathcal{C}_k\}_{k=1}^{K}$, any cluster with empirical majority purity $\mathrm{purity}(\mathcal{C}_k) < \tau$ (threshold $\tau \in (0, 1)$) is re-annotated with an uncertainty label $y_{\mathrm{uncert}} \coloneqq C+1$. This transforms $\mathcal{D}$ into $\mathcal{D}'$ over the augmented label space $\mathcal{Y}' = \mathcal{Y} \cup \{y_{\mathrm{uncert}}\}$.

\textbf{Phase 3 (Selective Classifier Retraining):} A final classifier $f_2: \mathcal{X} \to \mathcal{Y}'$ is retrained on $\mathcal{D}'$. During inference, predicting $y_{\mathrm{uncert}}$ triggers a deferred reject option for human review; otherwise, it outputs a clean prediction in $\mathcal{Y}$.

\paragraph{Foundational assumptions for feasibility analysis}
\label{subsec:assumptions2}

To analyze the feasibility boundary where GCUL strictly improves utility, we formalize the geometric properties within the ambient feature space $\mathbb{R}^{d_0}$ ($d_0 = 768$).

\begin{remark}[Space Independence]
Linear projections (LDA/PCA) are used solely for computational efficiency. Theoretical bounds are established in $\mathbb{R}^{d_0}$, as linear projections preserve centroid topological distances up to a bounded scale factor.
\end{remark}

\begin{assumption}[Error and Cost Bounds]
\label{ass:error_cost}
The base error satisfies $0 < \varepsilon < 1$. The normalized rejection cost $\lambda \coloneqq \eta_2 / \eta_1$ strictly satisfies $0 \le \lambda < 1$, where $\eta_1, \eta_2$ denote unit costs for misclassification and rejection, respectively.
\end{assumption}

\begin{assumption}[Statistical Consistency of Error Pool]
\label{ass:error_pool_approximation}
The empirical mean $\boldsymbol{\mu}_{\mathrm{err}}$ and directional variance $\sigma_{\mathrm{err}, c}^2$ computed over $\mathcal{E}$ serve as consistent estimates for the true confusion attractor centroid $\boldsymbol{\mu}_{\mathrm{mix}}$ and variance proxy $\sigma_{\mathrm{mix}, c}^2$ ($\boldsymbol{\mu}_{\mathrm{mix}} \approx \boldsymbol{\mu}_{\mathrm{err}}$, $\sigma_{\mathrm{mix}, c} \approx \sigma_{\mathrm{err}, c}$).
\end{assumption}

\begin{assumption}[Unimodal Single Attractor]
\label{ass:single_attractor}
The misclassified samples in $\mathcal{E}$ form a unimodal, dominant confusion attractor $\mathcal{C}_{\mathrm{conf}}$ centered at $\boldsymbol{\mu}_{\mathrm{conf}}$ in $\mathbb{R}^{d_0}$, rather than fragmenting into multiple disjoint sub-clusters.
\end{assumption}

\begin{assumption}[Attractor Initialization Stability]
\label{ass:initialization_stability}
The explicit warm-start $\boldsymbol{\mu}_{\mathrm{conf}}^{(0)}$ lies within the basin of attraction of $\mathcal{C}_{\mathrm{conf}}$, ensuring that the updated centroid $\boldsymbol{\mu}_{\mathrm{conf}}$ remains topologically stable during $K$-Means without collapsing into any clean class core.
\end{assumption}

\begin{assumption}[Asymptotic Gaussianity and Sub-Gaussian Dispersion]
\label{ass:asymptotic_gaussian_subgaussian}
In high-dimensional space $\mathbb{R}^{d_0}$ ($d_0 \gg 1$), for any clean class $c \in \mathcal{Y}$, the 1D marginal projection of feature $\mathbf{z}$ along $\hat{\mathbf{d}}_c = (\boldsymbol{\mu}_{\mathrm{conf}} - \boldsymbol{\mu}_c) / \|\boldsymbol{\mu}_{\mathrm{conf}} - \boldsymbol{\mu}_c\|_2$ asymptotically approaches a univariate Gaussian $\mathcal{N}(\mu_c, \sigma_c^2)$ \cite{diaconis1984asymptotics, klartag2007central}. These projections exhibit sub-Gaussian tail decay with variance proxies $\sigma_c^2, \sigma_{\mathrm{mix}, c}^2 > 0$, causing Voronoi tail probabilities to decay exponentially with centroid distance $D_c = \|\boldsymbol{\mu}_{\mathrm{conf}} - \boldsymbol{\mu}_c\|_2$.
\end{assumption}

\subsection{Selectivity score and monotonicity}
\label{subsec:selectivity_score}

To quantify GCUL's utility gain under a realistic operational cost structure, let $\eta_1 > 0$ and $\eta_2 \ge 0$ denote the unit costs of misclassification and rejection, respectively. Define the relative rejection cost as $\lambda \coloneqq \eta_2 / \eta_1 \in [0, 1)$. 

\textbf{Operational risk \& utility gain.} The baseline risk is $R_{\mathrm{base}} = N \varepsilon \eta_1$. Under GCUL, rejection is treated as a service failure event rather than a free option. Evaluating the conditional error rate $\varepsilon_{\mathrm{clean}}$ on accepted samples gives the operational risk $R_{\mathrm{GCUL}} =  \varepsilon_{\mathrm{clean}} \eta_1 +  \rho \eta_2$, where $\rho$ is the proportion of rejected samples. The normalized utility gain is:
\begin{equation}
\mathcal{U}(\lambda) \coloneqq \frac{R_{\mathrm{base}} - R_{\mathrm{GCUL}}}{\eta_1} = \varepsilon - \varepsilon_{\mathrm{clean}} - \lambda \rho.
\end{equation}

\textbf{Feasibility condition.} Expressing the rejection rate $\rho$ and conditional clean error rate $\varepsilon_{\mathrm{clean}}$ via the error capture rate $r_1 \coloneqq |\mathcal{C}_{\mathrm{conf}} \cap \mathcal{E}| / |\mathcal{E}|$ and clean contamination rate $r_2 \coloneqq |\mathcal{C}_{\mathrm{conf}} \setminus \mathcal{E}| / (N - |\mathcal{E}|)$, we have:
\begin{equation}
\rho = \varepsilon r_1 + (1-\varepsilon)r_2, \qquad \varepsilon_{\mathrm{clean}} = \frac{\varepsilon(1-r_1)}{1 - \rho}.
\end{equation}
Solving the strictly positive utility condition $\mathcal{U}(\lambda) > 0$ yields a data-driven effectiveness criterion $S_{\mathrm{GCUL}} > \lambda$, where $S_{\mathrm{GCUL}}$ is the \emph{GCUL Selectivity Score}:

{\scalefont{0.85}
\begin{equation}
\label{eq:selectivity_score}
S_{\mathrm{GCUL}} \coloneqq \frac{(1-\varepsilon)(r_1 - r_2)}{r_1 + \frac{1-\varepsilon}{\varepsilon}r_2 - \varepsilon r_1^2 - \frac{(1-\varepsilon)^2}{\varepsilon}r_2^2 - 2(1-\varepsilon)r_1r_2}.
\end{equation}
}
$S_{\mathrm{GCUL}}$ measures the net selectivity of the confusion attractor, normalized by feature distribution geometry. 

\begin{theorem}[Strict Monotonicity of Selectivity Score]
\label{thm:monotonicity}
For any $0 < \varepsilon < 1$ and $0 < r_2 < r_1 < 1$, $S_{\mathrm{GCUL}}$ is strictly monotonically increasing in $r_1$ and strictly monotonically decreasing in $r_2$:
\begin{equation}
\frac{\partial S_{\mathrm{GCUL}}}{\partial r_1} > 0, \qquad \frac{\partial S_{\mathrm{GCUL}}}{\partial r_2} < 0.
\end{equation}
\end{theorem}

Theorem~\ref{thm:monotonicity} confirms that increasing error pool capture ($r_1$) or reducing clean sample contamination ($r_2$) strictly expands GCUL's feasible operational regime. Complete algebraic derivations and proofs are deferred to Appendix~\ref{app:utility_proofs}.

\subsection{Effectiveness threshold and feasibility region}
\label{subsec:feasibility_threshold}

Having established the selectivity criterion $S_{\mathrm{GCUL}} > \lambda$, we now derive geometric closed-form bounds for the capture rate $r_1$ and contamination rate $r_2$ using high-dimensional measure concentration \citep{vershynin2018high}.

\textbf{1D Voronoi Reduction.} For each class $c \in [C]$, let $\mathbf{d}_c \coloneqq \boldsymbol{\mu}_{\mathrm{mix}} - \boldsymbol{\mu}_c$, $D_c \coloneqq \|\mathbf{d}_c\|_2$, and $\hat{\mathbf{d}}_c \coloneqq \mathbf{d}_c / D_c$ be the connecting unit direction. The decision boundary separating the confusion cluster $\mathcal{C}_{\mathrm{mix}}$ and class $c$ under $K$-Means Voronoi partitioning reduces to a 1D thresholding rule on the projected coordinate $z_c(x) \coloneqq \langle x - \boldsymbol{\mu}_{\mathrm{mix}}, \hat{\mathbf{d}}_c \rangle$: a sample is assigned to $\mathcal{C}_{\mathrm{mix}}$ relative to class $c$ if and only if $z_c(x) < D_c/2$.

\textbf{Concentration of directional projection measures.} Under Assumption~\ref{ass:asymptotic_gaussian_subgaussian}, the 1D marginal projections of the misclassified pool and clean class $c$ converge weakly to univariate Gaussians $\mathcal{N}(0, \sigma_{\mathrm{mix},c}^2)$ and $\mathcal{N}(-D_c, \sigma_c^2)$, respectively, where $\sigma_{\mathrm{mix},c}^2 = \hat{\mathbf{d}}_c^\top \Sigma_{\mathrm{mix}} \hat{\mathbf{d}}_c$ and $\sigma_c^2 = \hat{\mathbf{d}}_c^\top \Sigma_c \hat{\mathbf{d}}_c$ represent directional variance proxies.

\begin{lemma}[Boundary Loss and Contamination Probabilities]
\label{lem:projection_tails}
The boundary loss probability $p_c^{\mathrm{loss}}$ (a misclassified sample leaking out of $\mathcal{C}_{\mathrm{mix}}$) and clean contamination probability $q_c^{\mathrm{contam}}$ (a clean sample falsely captured into $\mathcal{C}_{\mathrm{mix}}$) are strictly quantified by Gaussian error function tails:
{\scalefont{0.9}
\begin{align}
p_c^{\mathrm{loss}} &\coloneqq \mathbb{P}\left\{ z_c > \frac{D_c}{2} \right\} = \frac{1}{2}\left[1 - \operatorname{erf}\left( \frac{D_c}{2\sqrt{2}\sigma_{\mathrm{mix},c}} \right)\right], \label{eq:p_loss_exact} \\
q_c^{\mathrm{contam}} &\coloneqq \mathbb{P}\left\{ z_c < -\frac{D_c}{2} \right\} = \frac{1}{2}\left[1 - \operatorname{erf}\left( \frac{D_c}{2\sqrt{2}\sigma_c} \right)\right]. \label{eq:q_contam_exact}
\end{align}
}

\end{lemma}

\begin{theorem}[Global Rate Bounds and Feasibility Threshold]
\label{thm:global_rate_bounds}
Let $\pi_c$ and $\varepsilon_c$ denote the prior probability and conditional base error rate of class $c$, respectively. The global error capture rate $r_1$ and clean contamination rate $r_2$ satisfy:
\begin{align}
r_1 &\ge 1 - \frac{1}{\varepsilon} \sum_{c=1}^C \pi_c \varepsilon_c \cdot \frac{1}{2}\left[1 - \operatorname{erf}\left(\frac{D_c}{2\sqrt{2}\,\sigma_{\mathrm{mix},c}}\right)\right], \label{eq:r1_bound} \\
r_2 &\le \frac{1}{1-\varepsilon} \sum_{c=1}^C \pi_c \cdot \frac{1}{2}\left[1 - \operatorname{erf}\left(\frac{D_c}{2\sqrt{2}\,\sigma_c}\right)\right]. \label{eq:r2_bound}
\end{align}
Consequently, under Assumptions 1–5, GCUL is guaranteed to be operationally effective ($S_{\mathrm{GCUL}} > \lambda$) whenever the geometric signal-to-noise ratios $D_c / \sigma_{\mathrm{mix},c}$ and $D_c / \sigma_c$ satisfy the condition evaluated on the RHS bounds of \eqref{eq:r1_bound} and \eqref{eq:r2_bound}.
\end{theorem}

Theorem~\ref{thm:global_rate_bounds} connects the operational utility directly to feature-space geometry: larger centroid separation $D_c$ relative to directional variance proxies $\sigma_{\mathrm{mix},c}, \sigma_c$ exponentially depresses boundary loss and contamination via $\operatorname{erf}(\cdot)$, ensuring $S_{\mathrm{GCUL}} > \lambda$. Detailed derivations and proofs are provided in Appendix~\ref{app:geometric_proofs}.

\begin{theorem}[Sufficient Condition for GCUL Feasibility]
\label{thm:gcul_effectiveness_sufficient_condition}
By the strict monotonicity of $S_{\mathrm{GCUL}}(r_1, r_2)$ (Theorem~\ref{thm:monotonicity}), substituting the global rate bounds $r_1 \ge \hat{r}_1$ and $r_2 \le \hat{r}_2$ into \eqref{eq:selectivity_score} yields $S_{\mathrm{GCUL}}(r_1, r_2) \ge S_{\mathrm{GCUL}}(\hat{r}_1, \hat{r}_2) \coloneqq \lambda^*$. A \textbf{sufficient condition} for GCUL to be operationally effective ($S_{\mathrm{GCUL}} > \lambda$) is:
\begin{equation}
\lambda^* > \lambda,
\end{equation}
where the critical rejection threshold $\lambda^*$ is evaluated at conservative rates:
\begin{align}
\hat{r}_1 &\coloneqq 1 - \frac{1}{2\varepsilon} \sum_{c=1}^C \pi_c \varepsilon_c \left[1 - \operatorname{erf}\left(\frac{D_c}{2\sqrt{2}\sigma_{\mathrm{mix},c}}\right)\right], \label{eq:r1_hat} \\
\hat{r}_2 &\coloneqq \frac{1}{2(1-\varepsilon)} \sum_{c=1}^C \pi_c \left[1 - \operatorname{erf}\left(\frac{D_c}{2\sqrt{2}\sigma_c}\right)\right]. \label{eq:r2_hat}
\end{align}
\end{theorem}

\subsection{Rigorous capture rate bounds via sub-Gaussian concentration}
\label{subsec:rigorous_r1_analysis}

To address potential structural instability when misclassified samples are sparsely dispersed across $\mathbb{R}^{d_0}$, we establish a rigorous class-wise lower bound for $r_1$ using sub-Gaussian concentration inequalities \citep{vershynin2018high} and Boole's inequality, without relying on heuristic outlier filters.

\textbf{Measure-theoretic error loss cover.} Let $\mathcal{E}_c \coloneqq \{i \in \mathcal{E} : y_i = c\}$ denote class $c$'s error pool ($|\mathcal{E}_c| = N \pi_c \varepsilon_c$). Error samples escaping the control of $\mathcal{C}_{\mathrm{mix}}$ are covered by two mechanisms: boundary truncation loss $\mathcal{E}_{c, \mathrm{boundary}}$ and extreme directional tail outliers $\mathcal{E}_{c, \mathrm{outlier}}$. 

Under sub-Gaussian concentration with variance proxy $\sigma_{\mathrm{mix},c}^2 = \hat{\mathbf{d}}_c^\top \Sigma_{\mathrm{mix}} \hat{\mathbf{d}}_c$, the Chernoff bound for the directional outlier fraction $\eta_{\mathrm{outlier}, c} \coloneqq |\mathcal{E}_{c, \mathrm{outlier}}| / |\mathcal{E}_c|$ yields:
\begin{equation}
\eta_{\mathrm{outlier}, c} \le \exp\left( -\frac{D_c^2}{8\sigma_{\mathrm{mix},c}^2} \right).
\label{eq:chernoff_outlier}
\end{equation}

\begin{theorem}[Rigorous Sub-Gaussian Lower Bound for Global Capture Rate $r_1$]
\label{thm:r1_rigorous}
Applying Boole's inequality over the non-disjoint loss cover $\mathcal{E}_{c, \mathrm{loss}} \subseteq \mathcal{E}_{c, \mathrm{boundary}} \cup \mathcal{E}_{c, \mathrm{outlier}}$, the empirical global capture rate $r_1$ satisfies the conservative lower bound:
{
\scalefont{0.7}
\begin{equation}
r_1 \ge 1 - \sum_{c=1}^C \frac{\pi_c \varepsilon_c}{\varepsilon} \left[ e^{\left( -\frac{D_c^2}{8\sigma_{\mathrm{mix},c}^2} \right)} + \frac{1}{2}\left[1 - \operatorname{erf}\left(\frac{D_c}{2\sqrt{2}\sigma_{\mathrm{mix},c}}\right)\right] \right]
\label{eq:r1_rigorous_bound}
\end{equation}
}
\end{theorem}

\begin{corollary}[Structural Confidence and Dense Convergence]
\label{cor:attractor_properties}
The non-outlier error fraction under structural attractor control, defined as Structural Existence Confidence $\mathcal{S}_{\mathrm{exist}} \coloneqq 1 - \sum_{c=1}^C \frac{\pi_c \varepsilon_c}{\varepsilon} \eta_{\mathrm{outlier},c}$, satisfies $\mathcal{S}_{\mathrm{exist}} \ge 1 - \sum_{c=1}^C \frac{\pi_c \varepsilon_c}{\varepsilon} \exp\left( -\frac{D_c^2}{8\sigma_{\mathrm{mix},c}^2} \right)$. In dense regimes where $\max_c (\sigma_{\mathrm{mix},c}/D_c) \to 0$, $\mathcal{S}_{\mathrm{exist}} \to 1$, naturally recovering the projection concentration bound in Theorem~\ref{thm:global_rate_bounds}.
\end{corollary}

Theorem~\ref{thm:r1_rigorous} provides an explicit safety margin against centroid collapse: when dispersion diverges ($\sigma_{\mathrm{mix},c} \to \infty$), the exponential term smoothly bounds the capture rate without unphysical measure collapse. Complete proofs are in Appendix~\ref{app:subgaussian_proofs}.

\vspace{-10pt}

% ==============================================================================
% Section 4: Empirical Evaluation
% ==============================================================================
\section{Empirical evaluation}
\label{sec:experiments}

\vspace{-5pt}

\subsection{Benchmark performance on real-world datasets}
\label{subsec:real_benchmarks}

To evaluate the empirical effectiveness of the proposed \mbox{GCUL} framework under a concentrated confusion attractor topology, we utilize a public multi-class emotion classification benchmark from Kaggle~\footnote{\url{https://www.kaggle.com/datasets/prajwalnayakat/text-emotion}}. The dataset comprises $106{,}355$ short text instances spanning 11 emotion categories with a balanced class distribution. Input texts are predominantly short-to-medium sequences ($\le 300$ characters), providing concise local emotional cues. Full dataset statistics, and class-wise decompositions are provided in Appendix~\ref{app:benchmarks_setup}.

The dataset is partitioned into training, validation, and test sets using an 80/10/10 ratio. We adopt accuracy and macro F1-score as standard metrics. For \mbox{GCUL} and its dimension-reduction variants, we additionally report the \textit{Uncertainty Rate} (i.e., the proportion of samples routed to the rejected/uncertain group) to evaluate selection efficiency.

Table~\ref{tab:results_1} presents the comparative performance between standard classification baselines (traditional machine learning and fine-tuned neural encoders) and our proposed \mbox{GCUL} framework (BiLSTM + Attention as benchmark ($f_1$) ). 
\vspace{-5pt}
\begin{table}[H]
\centering
\caption{Overall performance comparison of all models.}
\label{tab:results_1}
\scalebox{0.8}{
\begin{tabular}{l@{\hspace{0.1cm}}c@{\hspace{0.2cm}}c@{\hspace{0.2cm}}c}
\toprule
\textbf{Model} & \textbf{Accuracy (\%)} & \textbf{F1-score} & \textbf{Rej. Rate} \\
\midrule
Logistic Regression & 84.02 & 0.8433 & -- \\
Naive Bayes\cite{lewis1998naive}         & 72.80 & 0.7256 & -- \\
SVM\cite{joachims1998text}                 & 84.11 & 0.8421 & -- \\
Random Forest       & 84.75 & 0.8518 & -- \\
TextCNN\cite{kim2014convolutional}             & 86.35 & 0.8673 & -- \\
BiLSTM + Attention\cite{liu2019bidirectional,graves2005framewise}  & 87.83 & 0.8831 & -- \\
DistilBERT\cite{sanh2019distilbert}          & 89.37 & 0.8958 & -- \\
GCUL + PCA            & \textbf{95.85} & \textbf{0.9594} & 10.24\% \\
GCUL + LDA\cite{fisher1936use}            & 94.98 & 0.9504 & \textbf{8.49\%} \\
\bottomrule
\end{tabular}
}
\end{table}
\vspace{-15pt}
\paragraph*{Main findings:} 
1) \textbf{Substantial Accuracy Boost via Targeted Rejection:} The base classification backbone, DistilBERT ($f_1$), achieves an accuracy of $89.37\%$. By introducing the guided clustering rejection mechanism, \mbox{GCUL-LDA} elevates the selective accuracy on accepted samples to \textbf{$94.98\%$} (a net improvement of \textbf{$+5.61\%$}) while rejecting only $8.49\%$ highly ambiguous instances.
2) \textbf{Superiority of Supervised Dimensionality Reduction:} While \mbox{GCUL-PCA} achieves a marginal accuracy gain ($95.85\%$), \mbox{GCUL-LDA} effectively reduces the uncertainty rate from $10.24\%$ to $8.49\%$ (a $1.75\%$ absolute reduction in over-rejection). This validates that supervised projection via LDA better preserves class separability and minimizes intra-class variance in the embedding space before cluster routing.

Implementation details and training costs for all baselines are elaborated in Appendix~\ref{app:benchmarks_setup}.

\vspace{-10pt}
\subsection{Comparison with selective classification baselines}
\label{subsec:selective_exp}
\vspace{-5pt}
In this section, we compare GCUL with representative selective classification baselines to evaluate both its operational stability and performance across different dataset dynamics. 
Specifically, we consider three established mechanisms: 
(i) \textbf{MSP}~\cite{hendrycks2017baseline} paired with cost-dependent thresholding~\cite{chow1970optimum, geifman2017selective}, 
(ii) \textbf{SelectiveNet}~\cite{geifman2019selectivenet}, and 
(iii) \textbf{Mahalanobis-distance-based} rejection~\cite{lee2018simple}.
\vspace{-10pt}
\paragraph{Operating-point sensitivity of baselines vs. GCUL.}
We first analyze how selective classification performance varies with user-defined operating parameters. 
For MSP, we vary the rejection cost $\lambda$ to dynamically adapt the risk-minimizing confidence threshold. 
\vspace{-10pt}
\begin{figure}[H]
    \centering
    \includegraphics[width=0.36\textwidth]{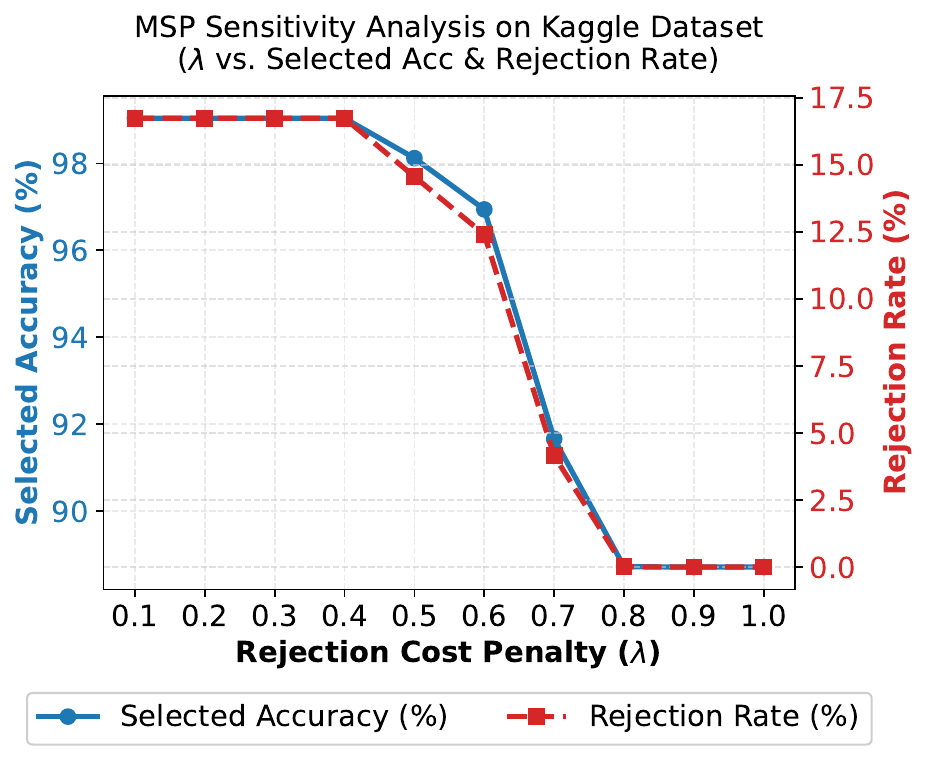}
    \includegraphics[width=0.36\textwidth]{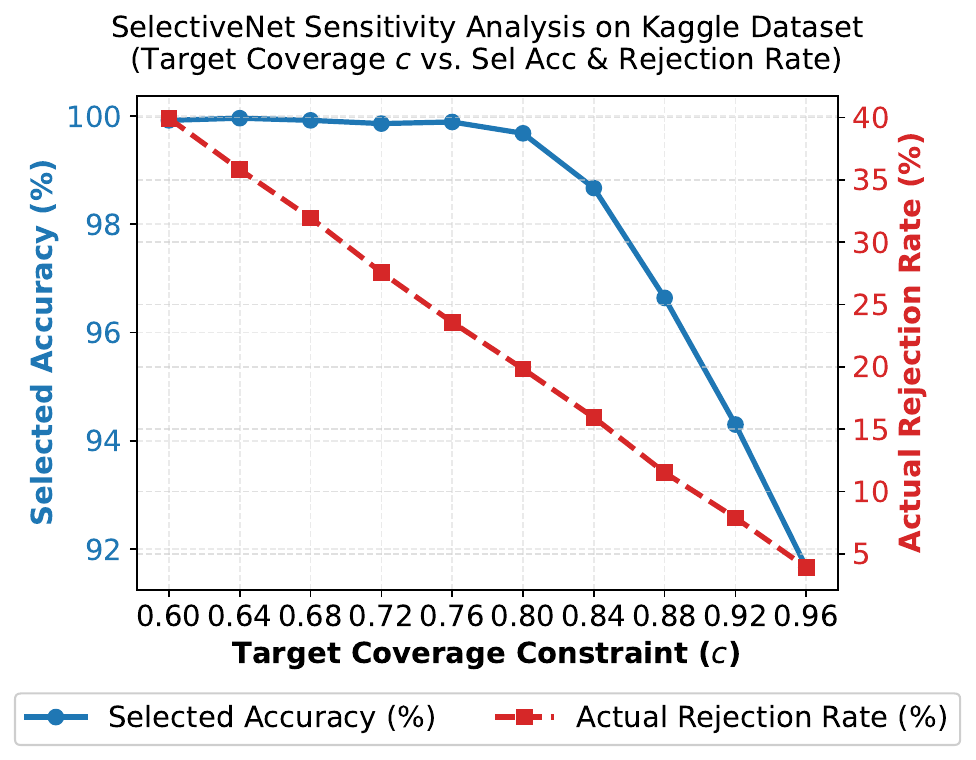}
    \includegraphics[width=0.36\textwidth]{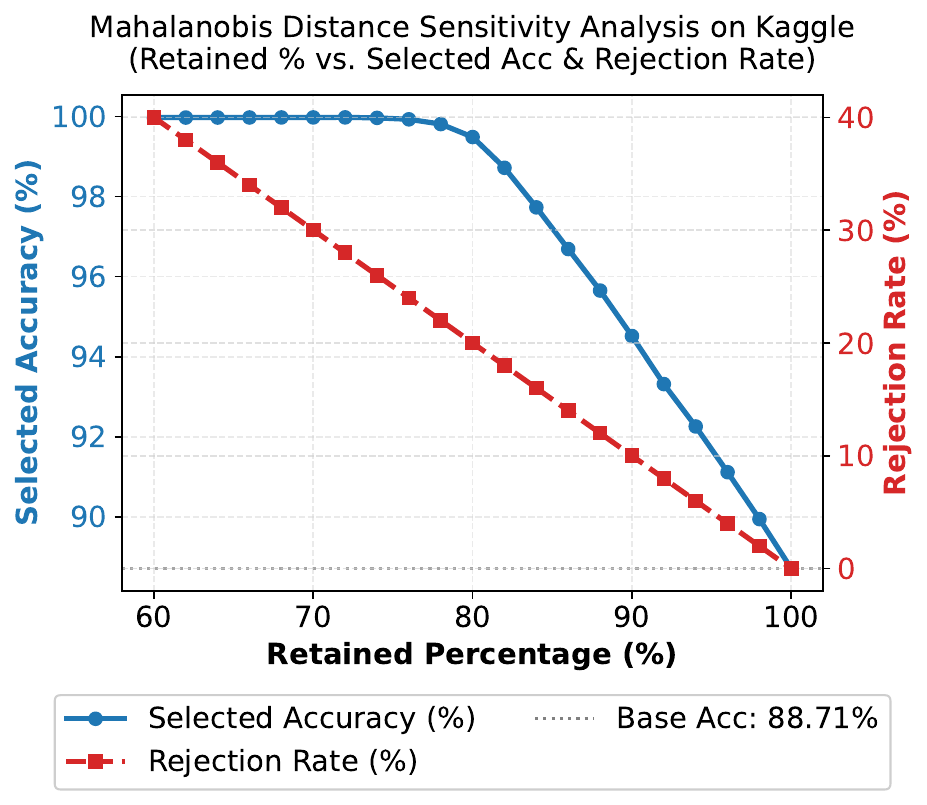}
    \caption{\textbf{Operating-point sensitivity of selective classification baselines on the Kaggle dataset.}
    Sensitivity of MSP to rejection cost $\lambda$, SelectiveNet to target coverage $c$, and Mahalanobis rejection to the retained percentage.}
    \label{fig:baseline_operating_sensitivity}
\end{figure}

\vspace{-10pt}
For SelectiveNet and Mahalanobis rejection, operating points are explicitly controlled via target coverage $c$ and retained percentage, respectively.

As illustrated in Figures~\ref{fig:baseline_operating_sensitivity}(a)--(c), traditional baselines provide flexible control over the prediction--rejection trade-off, but their behavior is highly sensitive to and rigidly parameterized by explicit external target quantities (e.g., coverage or rejection cost). 
In contrast, GCUL achieves competitive selective performance without requiring a manually prescribed rejection fraction. Instead, it intrinsically determines the rejection boundary from the geometric structure of the representation space via cluster-guided uncertainty identification, yielding superior operational stability across reasonable parameter regimes.

\paragraph{Sensitivity of GCUL to representation dimensionality.}
Since GCUL relies on representation geometry, we next investigate its operational stability under varying representation dimensionalities. 
Specifically, we apply PCA to the BERT embeddings and evaluate GCUL across a wide spectrum of dimensions.

\begin{figure}[H]
    \centering
    \includegraphics[width=0.48\textwidth]{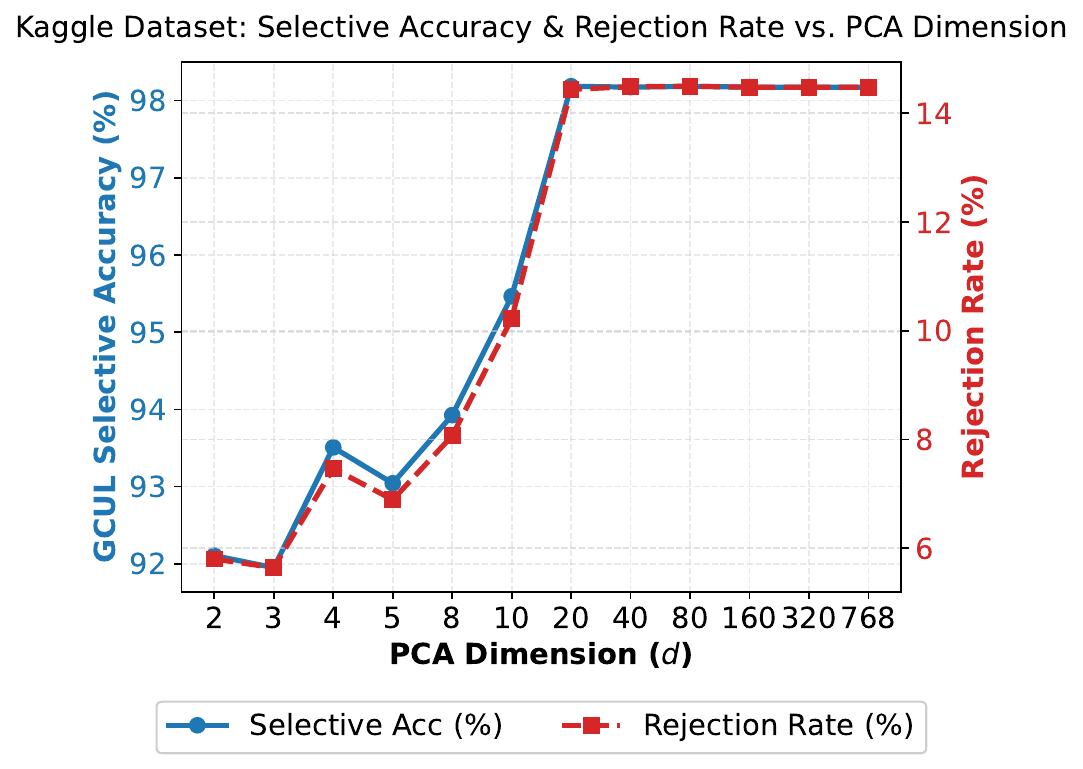}
    \caption{\textbf{Sensitivity of GCUL to representation dimensionality on the Kaggle dataset.}}
    \label{fig:gcul_pca_sweep_kaggle_acc_rej}
\end{figure}

As shown in Figure~\ref{fig:gcul_pca_sweep_kaggle_acc_rej}, GCUL exhibits higher variance in extremely low-dimensional regimes where critical geometric structure is lost. 
However, as the dimensionality increases, its selective accuracy and rejection rate stabilize over a broad range. 
This confirms that while GCUL replaces explicit target coverage tuning with geometric clustering, it remains robust as long as the representation space retains sufficient capacity to preserve local boundary structures.

\paragraph{Validation of the theoretical feasibility boundary.}
To understand whether GCUL's theoretical guarantees align with this geometric stability, we compare the predicted critical rejection cost $\lambda^*_{\mathrm{PRE}}$ against the empirical utility threshold $\lambda^*_{\mathrm{EMP}}$ across PCA dimensions.

Figure~\ref{fig:gcul_pca_sweep_kaggle_lambda} demonstrates that $\lambda^*_{\mathrm{PRE}}$ tracks the trend of the empirical feasibility region while consistently acting as a conservative lower bound (due to the use of probabilistic upper bounds in our theoretical derivations). 
Importantly, this theoretical criterion can be evaluated \textit{a priori} before deploying the full pipeline, providing a practical decision tool to verify whether a given representation space is suitable for GCUL rejection.

\begin{figure}[H]
    \centering
    \includegraphics[width=0.45\textwidth]{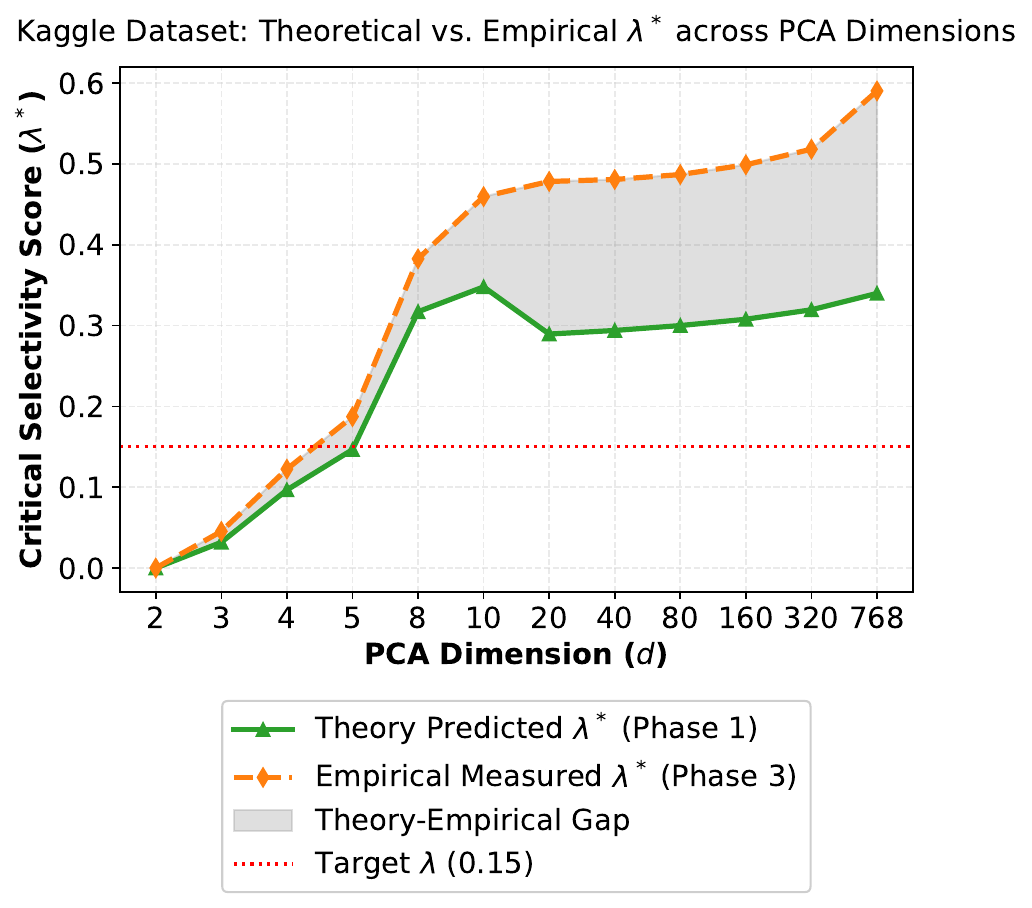}
    \caption{\textbf{Theoretical and empirical critical rejection costs across representation dimensions.}}
    \label{fig:gcul_pca_sweep_kaggle_lambda}
\end{figure}

\paragraph{Cross-dataset robustness and failure case screening.}
We further evaluate all selective classification methods across three datasets (Kaggle, GoEmotions\cite{demszky2020goemotions}, Dair-AI~\footnote{\url{https://huggingface.co/dair-ai/datasets}} \cite{saravia2018carer}) under multiple PCA dimensions and rejection-cost settings. We report the representative qualitative patterns in the main text, while providing the complete experimental results in Appendix~\ref{app:comparison_result}.

The three datasets exhibit substantially different behaviors. On the Kaggle
dataset, GCUL demonstrates a relatively stable utility--rejection trade-off
across different PCA dimensions, with its rejection region determined by the
underlying embedding geometry rather than a manually specified rejection
fraction. On the other hand, the GoEmotions dataset represents a challenging
failure case for selective classification. In this setting, none of the
evaluated methods provides a meaningful utility improvement. In particular,
the MSP-based method exhibits an almost degenerate behavior: for
$\lambda < 0.6$, it rejects nearly all samples, whereas for
$\lambda > 0.6$, it rejects almost none. The other baseline methods similarly
follow their predefined selection mechanisms without producing a positive
utility gain. GCUL also fails to improve the overall utility in this setting;
however, its theoretical feasibility analysis correctly identifies the
absence of a feasible operating regime, yielding
$\lambda^*_{\mathrm{PRE}} = 0$. This indicates that the theoretical analysis
can serve as a useful pre-deployment screening mechanism, allowing
practitioners to identify potentially unsuitable datasets without requiring
a complete selective-classification training and evaluation pipeline.

\subsection{Synthetic phase-transition verification}
\label{subsec:synthetic_exp}
To evaluate whether the theoretical feasibility criterion captures the degradation of GCUL under controlled perturbations, we conduct a synthetic experiment in which the geometry of a GCUL-favorable dataset is progressively corrupted. The purpose of this experiment is not to reproduce a particular real-world data distribution, but to provide a controlled environment for examining the relationship between the predicted critical rejection cost ($\lambda^*_{\mathrm{PRE}}$) and its empirical counterpart ($\lambda^*_{\mathrm{EMP}}$).

\paragraph{Synthetic setup and compound noise injection.} 
We construct a 10-dimensional feature space containing five clean class populations  (2,000 samples per class) and an additional 2,000-sample localized confusion population. In the zero-noise regime, the baseline classifier achieves approximately 86.5\% accuracy, while GCUL achieves 100\% selective accuracy with a refusal rate of approximately 16.5\%

We introduce two forms of corruption simultaneously: isotropic Gaussian feature noise and random label corruption,
\[
\sigma_{\mathrm{noise}} = \eta \\ p_{\mathrm{flip}} = 0.025\eta
\]
where $\eta$ is varied from 0 to 16 over 41 levels. For each noise level, the experiment is repeated with 10 random seeds, and the reported trajectories are averaged across seeds. Single run result was provided in Appendix~\ref{app:synthetic_result}

\begin{figure}[H]
    \centering
    \includegraphics[width=0.48\textwidth]{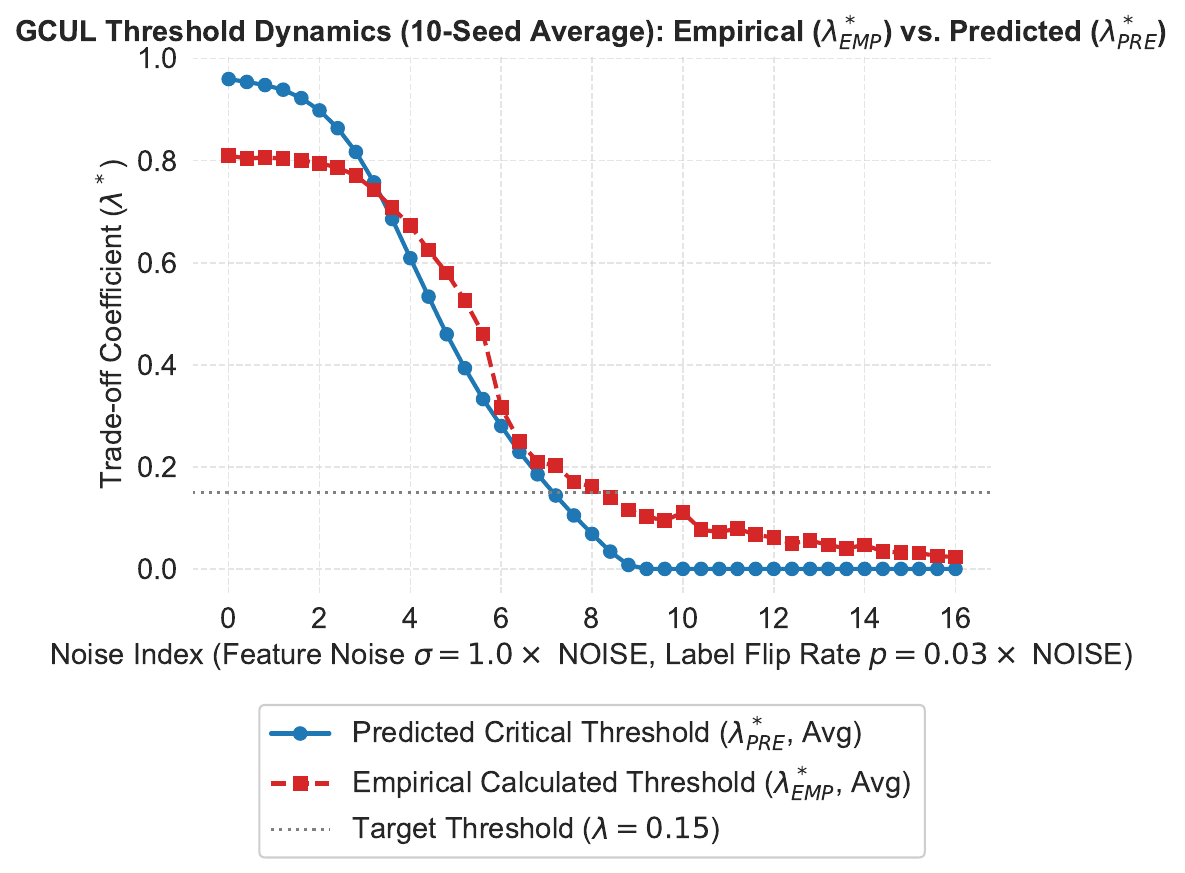}
    \caption{\textbf{Dynamics of critical trade-off thresholds ($\lambda^*$) across synthetic noise regimes (10-seed average).}}
    \label{fig:synthetic_dynamics}
\end{figure}

\paragraph{Phase-transition dynamics and threshold alignment.}
Figure~\ref{fig:synthetic_dynamics} illustrates the dynamic trajectories of critical trade-off thresholds as a function of the Noise Index. Key analytical observations include:

\begin{enumerate}
    \item \textbf{Tight bound alignment:} Across the entire noise sweep, the theoretical prediction $\lambda^*_{\mathrm{PRE}}$ closely tracks the empirical threshold $\lambda^*_{\mathrm{EMP}}$. Particularly during the critical decay phase ($\mathrm{Noise} \in [3.0, 7.0]$), both curves exhibit a near-identical monotonic decrease, validating that our derived sufficient condition provides a tight, non-loose theoretical bound for selective utility.
    
    \item \textbf{Asymptotic decay under heavy corruption:} As noise exceeds $8.5$, $\lambda^*_{\mathrm{PRE}}$ asymptotes to $0.0$, indicating that isotropic noise has completely polluted the cluster geometry, rendering selective rejection mathematically unviable. The empirical curve $\lambda^*_{\mathrm{EMP}}$ retains a minor residual floor due to discrete sample partitioning, but offers no effective utility gain over the baseline.
\end{enumerate}

\paragraph{Reliability and safety envelope analysis.}
While averaging across $10$ seeds yields smooth macro-trajectories, individual empirical runs ($\lambda^*_{\mathrm{EMP}}$) naturally exhibit local variance due to stochastic data sampling. To evaluate the mathematical integrity and safety guarantees of our theoretical feasibility conditions, we analyze micro-level predictions across all $410$ individual evaluation runs ($41$ noise steps $\times$ $10$ random seeds). We categorize discrepancies into two violation modes:
\begin{itemize}
    \item \textbf{Type-1 Violation (Conservative Error):} Theory predicts \textit{Infeasible} ($\lambda^*_{\mathrm{PRE}} < \lambda$), but empirical execution yields \textit{Effective} ($\lambda^*_{\mathrm{EMP}} \ge \lambda$).
    \item \textbf{Type-2 Violation (Fatal Risk):} Theory predicts \textit{Feasible} ($\lambda^*_{\mathrm{PRE}} \ge \lambda$), but empirical execution \textit{Fails} ($\lambda^*_{\mathrm{EMP}} < \lambda$).
\end{itemize}

\begin{table}[H]
\centering
\small
\caption{\textbf{Quantitative reliability analysis of theoretical feasibility predictions across 410 individual runs (41 noise levels $\times$ 10 random seeds).}}
\label{tab:violation_analysis}
\begin{tabular}{lll}
\toprule
\textbf{Violation Type} & \textbf{Count / Total} & \textbf{Rate (\%)} \\
\midrule
\textbf{Type-1} & 25 / 410 & 6.1\% \\
\textbf{Type-2} & \textbf{0 / 410} & \textbf{0\%} \\
\midrule
\textbf{Overall} & 25 / 410 & 6.1\% \\
\bottomrule
\end{tabular}
\end{table}

Across all 410 runs, no Type-2 violation is observed, whereas 25 runs exhibit Type-1 violations. This asymmetric error pattern is consistent with the intended role of the condition as a sufficient rather than necessary criterion: when the condition is satisfied, we observe no empirical failure in this controlled experiment, while failure of the condition does not necessarily imply that GCUL is ineffective.

% ==============================================================================
% Section 5: Conclusion and Discussion
% ==============================================================================

\section{Conclusion and discussion}
\label{sec:conclusion}

This work introduces Guided Clustering-based Uncertain Learning (GCUL), a geometry-driven framework for selective classification that identifies confusion regions directly from representation spaces. By establishing a theoretical feasibility criterion based on high-dimensional measure concentration, GCUL enables pre-deployment screening to assess applicability before executing full operating-point searches. Our empirical evaluations confirm that GCUL substantially improves selective performance on datasets exhibiting coherent confusion structures, highlighting its role as a geometry-dependent selective learning framework rather than a universal baseline.

In addition to its predictive performance, GCUL offers distinct practical advantages for deployment. It operates as an architectural wrapper around standard classifiers without altering backbone designs, allowing Phase~1 models and representations to be fully reused. Furthermore, across various classification backbones, GCUL exhibits remarkably consistent qualitative behavior, suggesting that localized hard-example clustering reflects intrinsic task-distribution geometry rather than model-specific artifacts.

Despite these strengths, several limitations remain. The performance of GCUL depends on the choice of representation dimensionality, where overly aggressive reduction risks losing spatial information while excessively high dimensions yield conservative clustering. Beyond that, Although GCUL is formulated as a general framework for multi-class classification and does not intrinsically depend on a particular data modality, the current empirical evaluation is restricted to text-based emotion classification. Additionally, our theoretical guarantees rest on structural assumptions, such as a dominant confusion attractor, which may be violated when unconfident samples form multiple disconnected clusters or exhibit severe non-Gaussianity.

Future research will focus on addressing these bounds. Key directions include developing automated dimension selection rules, extending the single-attractor formulation to multimodal or hierarchical confusion clusters, and evaluating GCUL across computer vision and multimodal domains to verify the broader generalizability of representation geometry in selective learning.

% ==============================================================================
% References
% ==============================================================================
\bibliographystyle{plain}
\bibliography{ref}

@article{diaconis1984asymptotics,
  title={Asymptotics of graphical projection pursuit},
  author={Diaconis, Persi and Freedman, David},
  journal={The Annals of Statistics},
  volume={12},
  number={3},
  pages={793--815},
  year={1984},
  publisher={Institute of Mathematical Statistics},
  doi={10.1214/aos/1176346703}
}

@article{klartag2007central,
  title={A central limit theorem for convex sets},
  author={Klartag, Bo'az},
  journal={Inventiones Mathematicae},
  volume={168},
  number={1},
  pages={91--131},
  year={2007},
  publisher={Springer},
  doi={10.1007/s00222-006-0028-8}
}

@article{chow1970optimum,
  title={On optimum recognition error and reject tradeoff},
  author={Chow, C. K.},
  journal={IEEE Transactions on Information Theory},
  volume={16},
  number={1},
  pages={41--46},
  year={1970},
  doi={10.1109/TIT.1970.1054406}
}

@article{elyaniv2010foundations,
  author  = {Ran El-Yaniv and Yair Wiener},
  title   = {On the Foundations of Noise-free Selective Classification},
  journal = {Journal of Machine Learning Research},
  year    = {2010},
  volume  = {11},
  number  = {53},
  pages   = {1605-1641},
  url     = {http://jmlr.org/papers/v11/el-yaniv10a.html}
}

@book{vershynin2018high,
  title={High-Dimensional Probability: An Introduction with Applications in Data Science},
  author={Vershynin, Roman},
  series={Cambridge Series in Statistical and Probabilistic Mathematics},
  volume={47},
  publisher={Cambridge University Press},
  year={2018},
  doi={10.1017/9781108231596}
}

@inproceedings{hendrycks2017baseline,
  title     = {A Baseline for Detecting Misclassified and Out-of-Distribution Examples in Neural Networks},
  author    = {Hendrycks, Dan and Gimpel, Kevin},
  booktitle = {International Conference on Learning Representations},
  year      = {2017},
  url       = {https://openreview.net/forum?id=Hkg4TI9xl}
}

@inproceedings{guo2017calibration,
  title={On calibration of modern neural networks},
  author={Guo, Chuan and Pleiss, Geoff and Sun, Yu and Weinberger, Kilian Q.},
  booktitle={Proceedings of the 34th International Conference on Machine Learning},
  volume={70},
  pages={1321--1330},
  year={2017},
  publisher={JMLR.org},
  doi={10.5555/3305381.3305518}
}

@inproceedings{jiang2018trustscore,
  title={To trust or not to trust a classifier},
  author={Jiang, Heinrich and Kim, Been and Guan, Melody Y. and Gupta, Maya},
  booktitle={Proceedings of the 32nd International Conference on Neural Information Processing Systems},
  pages={5546--5557},
  year={2018},
  publisher={Curran Associates Inc.},
  doi={10.5555/3327345.3327458}
}

@inproceedings{shrivastava2016ohem,
  title={Training region-based object detectors with online hard example mining},
  author={Shrivastava, Abhinav and Gupta, Abhinav and Girshick, Ross},
  booktitle={Proceedings of the IEEE Conference on Computer Vision and Pattern Recognition},
  pages={761--769},
  year={2016},
  doi={10.1109/CVPR.2016.89}
}

@inproceedings{lin2017focalloss,
  title     = {Focal loss for dense object detection},
  author    = {Lin, Tsung-Yi and Goyal, Priya and Girshick, Ross and He, Kaiming and Doll{\'a}r, Piotr},
  booktitle = {Proceedings of the IEEE International Conference on Computer Vision},
  pages     = {2980--2988},
  year      = {2017},
  doi       = {10.1109/ICCV.2017.324}
}

@inproceedings{swayamdipta2020cartography,
  title     = {Dataset cartography: Mapping and diagnosing datasets with training dynamics},
  author    = {Swayamdipta, Swabha and Schwartz, Roy and Lourie, Nicholas and Wang, Yizhong and Hajishirzi, Hannaneh and Smith, Noah A. and Choi, Yejin},
  booktitle = {Proceedings of the 2020 Conference on Empirical Methods in Natural Language Processing},
  pages     = {9275--9293},
  year      = {2020},
  publisher = {Association for Computational Linguistics},
  doi       = {10.18653/v1/2020.emnlp-main.746}
}

@inproceedings{robinson2021hardnegative,
  title     = {Contrastive learning with hard negative samples},
  author    = {Robinson, Joshua and Chuang, Ching-Yao and Sra, Suvrit and Jegelka, Stefanie},
  booktitle = {International Conference on Learning Representations},
  year      = {2021},
  url       = {https://openreview.net/forum?id=CR1XOQ0UTh-}
}

@inproceedings{sun2022knn,
  title     = {Out-of-distribution Detection with Deep Nearest Neighbors},
  author    = {Sun, Yiyou and Ming, Yifei and Zhu, Xiaojin and Li, Yixuan},
  booktitle = {Proceedings of the 39th International Conference on Machine Learning},
  volume    = {162},
  pages     = {20827--20840},
  year      = {2022},
  publisher = {PMLR},
  doi       = {10.5555/3546258.3546424},
  url       = {https://proceedings.mlr.press/v162/sun22a.html}
}

@inproceedings{ziyin2019deepgamblers,
  title     = {Deep gamblers: Learning to abstain with portfolio theory},
  author    = {Ziyin, Liu and Wang, Zhikang and Ueda, Masahito},
  booktitle = {Advances in Neural Information Processing Systems},
  volume    = {32},
  pages     = {10623--10633},
  year      = {2019},
  url       = {https://proceedings.neurips.cc/paper/2019/hash/83693e50838152eb9b3a3c94291e1d3b-Abstract.html}
}

@inproceedings{liang2018enhancing,
  title     = {Enhancing the reliability of out-of-distribution image detection in neural networks},
  author    = {Liang, Shiyu and Li, Yixuan and Srikant, R.},
  booktitle = {International Conference on Learning Representations},
  year      = {2018},
  url       = {https://openreview.net/forum?id=H1VGkIxRZ}
}

@inproceedings{saravia2018carer,
  title     = {{CARER}: Contextualized affect representations for emotion recognition},
  author    = {Saravia, Elvis and Liu, Hsien-Chi Toby and Huang, Yen-Hao and Wu, Junlin and Chen, Yi-Shin},
  booktitle = {Proceedings of the 2018 Conference on Empirical Methods in Natural Language Processing},
  pages     = {3685--3695},
  year      = {2018},
  publisher = {Association for Computational Linguistics},
  doi       = {10.18653/v1/D18-1404}
}

@inproceedings{demszky2020goemotions,
  title     = {{G}o{E}motions: A dataset of fine-grained emotions},
  author    = {Demszky, Dorottya and Movshovitz-Attias, Dana and Ko, Jeongwoo and Cowen, Alan and Nemade, Gaurav and Ravi, Sujith},
  booktitle = {Proceedings of the 58th Annual Meeting of the Association for Computational Linguistics},
  pages     = {4040--4054},
  year      = {2020},
  publisher = {Association for Computational Linguistics},
  doi       = {10.18653/v1/2020.acl-main.372}
}

@inproceedings{lee2018simple,
  title     = {A simple unified framework for detecting out-of-distribution samples and adversarial attacks},
  author    = {Lee, Kimin and Lee, Kibok and Lee, Honglak and Shin, Jinwoo},
  booktitle = {Advances in Neural Information Processing Systems},
  volume    = {31},
  pages     = {7167--7177},
  year      = {2018},
  doi       = {10.5555/3327757.3327819}
}

@inproceedings{geifman2019selectivenet,
  title     = {{S}elective{N}et: A deep neural network with an integrated reject option},
  author    = {Geifman, Yonatan and El-Yaniv, Ran},
  booktitle = {Proceedings of the 36th International Conference on Machine Learning},
  volume    = {97},
  pages     = {2151--2159},
  year      = {2019},
  publisher = {PMLR},
  url       = {https://proceedings.mlr.press/v97/geifman19a.html}
}

@inproceedings{geifman2017selective,
  title     = {Selective classification for deep neural networks},
  author    = {Geifman, Yonatan and El-Yaniv, Ran},
  booktitle = {Advances in Neural Information Processing Systems},
  volume    = {30},
  pages     = {4885--4894},
  year      = {2017},
  doi       = {10.5555/3295222.3295241}
}

@inproceedings{devlin2019bert,
  title     = {{BERT}: Pre-training of deep bidirectional transformers for language understanding},
  author    = {Devlin, Jacob and Chang, Ming-Wei and Lee, Kenton and Toutanova, Kristina},
  booktitle = {Proceedings of the 2019 Conference of the North American Chapter of the Association for Computational Linguistics: Human Language Technologies},
  pages     = {4171--4186},
  year      = {2019},
  publisher = {Association for Computational Linguistics},
  doi       = {10.18653/v1/N19-1423}
}

@inproceedings{joachims1998text,
  title     = {Text categorization with support vector machines: Learning with many relevant features},
  author    = {Joachims, Thorsten},
  booktitle = {Proceedings of the European Conference on Machine Learning},
  pages     = {137--142},
  year      = {1998},
  doi       = {10.1007/BFb0026683}
}

@inproceedings{kim2014convolutional,
  title     = {Convolutional neural networks for sentence classification},
  author    = {Kim, Yoon},
  booktitle = {Proceedings of the 2014 Conference on Empirical Methods in Natural Language Processing},
  pages     = {1746--1751},
  year      = {2014},
  publisher = {Association for Computational Linguistics},
  doi       = {10.3115/v1/D14-1181}
}

@unpublished{sanh2019distilbert,
  title  = {{D}istil{BERT}, a distilled version of {BERT}: smaller, faster, cheaper and lighter},
  author = {Sanh, Victor and Debut, Lysandre and Chaumond, Julien and Wolf, Thomas},
  year   = {2019},
  note   = {arXiv:1910.01108},
  doi    = {10.48550/arXiv.1910.01108}
}

@inproceedings{lewis1998naive,
  title     = {Naive ({B}ayes) at forty: The independence assumption in information retrieval},
  author    = {Lewis, David D.},
  booktitle = {Proceedings of the European Conference on Machine Learning},
  pages     = {4--15},
  year      = {1998},
  note      = {Invited talk},
  doi       = {10.1007/BFb0026666}
}

@article{fisher1936use,
  title     = {The use of multiple measurements in taxonomic problems},
  author    = {Fisher, Ronald A.},
  journal   = {Annals of Eugenics},
  volume    = {7},
  number    = {2},
  pages     = {179--188},
  year      = {1936},
  doi       = {10.1111/j.1469-1809.1936.tb02137.x}
}

@inproceedings{graves2005framewise,
  title     = {Framewise phoneme classification with bidirectional {LSTM} networks},
  author    = {Graves, Alex and Schmidhuber, J{\"u}rgen},
  booktitle = {Proceedings of the 2005 IEEE International Joint Conference on Neural Networks},
  volume    = {4},
  pages     = {2047--2052},
  year      = {2005},
  publisher = {IEEE},
  doi       = {10.1109/IJCNN.2005.1556215}
}

@article{liu2019bidirectional,
  title     = {Bidirectional {LSTM} with attention mechanism and convolutional layer for text classification},
  author    = {Liu, Gang and Guo, Jiabao},
  journal   = {Neurocomputing},
  volume    = {337},
  pages     = {325--338},
  year      = {2019},
  doi       = {10.1016/j.neucom.2019.01.078}
}

% ==============================================================================
% Appendix
% ==============================================================================
\newpage
\appendix
\section*{Appendix}
\addcontentsline{toc}{section}{Appendix}

\section{Proofs and derivations for section~\ref{subsec:selectivity_score}}
\label{app:utility_proofs}

\subsection{Derivation of the effectiveness criterion (Eq.~\ref{eq:selectivity_score})}
\label{app:criterion_derivation}

Starting from the positive utility condition $\mathcal{U}(\lambda) > 0$, we have:
\begin{equation}
\varepsilon_{\mathrm{clean}} + \lambda\rho < \varepsilon \implies \frac{\varepsilon(1-r_1)}{1-\rho} + \lambda\rho < \varepsilon.
\end{equation}
Dividing both sides by $\varepsilon > 0$ yields:
\begin{equation}
\frac{1-r_1}{1-\rho} < 1 - \frac{\lambda}{\varepsilon}\rho \implies 1-r_1 < (1-\rho)\left(1 - \frac{\lambda}{\varepsilon}\rho\right).
\end{equation}
Expanding the right-hand side:
{
{\scalefont{0.93}
\begin{equation}
1-r_1 < 1 - \rho - \frac{\lambda}{\varepsilon}\rho + \frac{\lambda}{\varepsilon}\rho^2 \implies -r_1 < -\rho + \frac{\lambda}{\varepsilon}\rho(\rho-1).
\end{equation}
}
Rearranging terms:
\begin{equation}
\rho - r_1 > \frac{\lambda}{\varepsilon}\rho(1-\rho).
\end{equation}
Substituting $\rho = \varepsilon r_1 + (1-\varepsilon)r_2$, the left-hand side simplifies to:
\begin{equation}
\rho - r_1 = (1-\varepsilon)r_2 - (1-\varepsilon)r_1 = (1-\varepsilon)(r_2 - r_1).
\end{equation}
Thus, we obtain:
{
\scalefont{0.88}
\begin{equation}
(1-\varepsilon)(r_1 - r_2) > \frac{\lambda}{\varepsilon}\rho(1-\rho) \implies \lambda < \frac{\varepsilon(1-\varepsilon)(r_1 - r_2)}{\rho(1-\rho)}.
\end{equation}
}
Expanding $\rho(1-\rho)$ explicitly:
{
\scalefont{0.88}
\begin{align}
&\rho(1-\rho) \\&= [\varepsilon r_1 + (1-\varepsilon)r_2][1 - \varepsilon r_1 - (1-\varepsilon)r_2] \nonumber \\
&= \varepsilon r_1 + (1-\varepsilon)r_2 - \varepsilon^2 r_1^2 - (1-\varepsilon)^2 r_2^2 - 2\varepsilon(1-\varepsilon)r_1 r_2.
\end{align}
}
Dividing the numerator and denominator by $\varepsilon$ yields the closed-form expression for $S_{\mathrm{GCUL}}$ in Eq.~\ref{eq:selectivity_score}. \hfill $\square$

\subsection{Proof of Theorem~\ref{thm:monotonicity} (strict monotonicity)}
\label{app:proof_monotonicity}

Let $a = 1-\varepsilon$. We write $S_{\mathrm{GCUL}} = \frac{a(r_1 - r_2)}{D}$, where $D = r_1 + \frac{a}{\varepsilon}r_2 - \varepsilon r_1^2 - \frac{a^2}{\varepsilon}r_2^2 - 2a r_1 r_2$.

\textbf{Part 1: Derivative w.r.t. $r_1$.}
\begin{equation}
\frac{\partial S}{\partial r_1} = \frac{a}{D^2} \left[ D - (r_1 - r_2)\frac{\partial D}{\partial r_1} \right].
\end{equation}
Calculating $\frac{\partial D}{\partial r_1} = 1 - 2\varepsilon r_1 - 2a r_2$ and substituting:
{
\begin{align*}
&D - (r_1 - r_2)\frac{\partial D}{\partial r_1} \\&=\left(r_1 + \frac{a}{\varepsilon}r_2 - \varepsilon r_1^2 - \frac{a^2}{\varepsilon}r_2^2 - 2a r_1 r_2\right) \\ & - (r_1 - r_2)(1 - 2\varepsilon r_1 - 2a r_2) \nonumber \\
&=\varepsilon(r_1 - r_2)^2 + \frac{r_2}{\varepsilon}(1 - r_2).
\end{align*}
}
Since $\varepsilon \in (0,1)$, $D > 0$, and $0 < r_2 < r_1 < 1$, both terms are strictly positive, implying $\frac{\partial S}{\partial r_1} > 0$.

\textbf{Part 2: Derivative w.r.t. $r_2$.}
\begin{equation}
\frac{\partial S}{\partial r_2} = \frac{a}{D^2} \left[ -D - (r_1 - r_2)\frac{\partial D}{\partial r_2} \right].
\end{equation}
Calculating $\frac{\partial D}{\partial r_2} = \frac{a}{\varepsilon} - \frac{2a^2}{\varepsilon}r_2 - 2a r_1$ and simplifying:
{
\scalefont{0.93}
\begin{equation}
-D - (r_1 - r_2)\frac{\partial D}{\partial r_2} = -\left[ \frac{a^2}{\varepsilon}(r_1 - r_2)^2 + \frac{r_1}{\varepsilon}(1 - r_1) \right].
\end{equation}
}
Since all terms within the brackets are non-negative and $r_1 > 0$, we conclude $\frac{\partial S}{\partial r_2} < 0$. \hfill $\square$

\section{Proofs for geometric modeling and rate bounds}
\label{app:geometric_proofs}

\subsection{Proof of 1D Voronoi projection reduction}
\label{app:proof_voronoi}

\begin{proof}
Under $K$-Means clustering, a sample $x \in \mathbb{R}^{d_0}$ is assigned to $\mathcal{C}_{\mathrm{mix}}$ over class center $\boldsymbol{\mu}_c$ if and only if it is strictly closer to $\boldsymbol{\mu}_{\mathrm{mix}}$ in Euclidean distance:
\begin{equation}
\|x - \boldsymbol{\mu}_{\mathrm{mix}}\|_2^2 < \|x - \boldsymbol{\mu}_c\|_2^2.
\end{equation}
Substituting $\boldsymbol{\mu}_c = \boldsymbol{\mu}_{\mathrm{mix}} - \mathbf{d}_c$ into the right-hand side:
\begin{align}
&\|x - \boldsymbol{\mu}_c\|_2^2 = \|(x - \boldsymbol{\mu}_{\mathrm{mix}}) + \mathbf{d}_c\|_2^2 \nonumber \\
&= \|x - \boldsymbol{\mu}_{\mathrm{mix}}\|_2^2 + 2\langle x - \boldsymbol{\mu}_{\mathrm{mix}}, \mathbf{d}_c \rangle + \|\mathbf{d}_c\|_2^2.
\end{align}
Subtracting $\|x - \boldsymbol{\mu}_{\mathrm{mix}}\|_2^2$ from both sides yields:
{\scalefont{0.85}
\begin{equation}
0 < 2\langle x - \boldsymbol{\mu}_{\mathrm{mix}}, \mathbf{d}_c \rangle + D_c^2 \implies 0 < 2 D_c \langle x - \boldsymbol{\mu}_{\mathrm{mix}}, \hat{\mathbf{d}}_c \rangle + D_c^2.
\end{equation}
}
Dividing by $2 D_c > 0$ gives $z_c(x) \coloneqq \langle x - \boldsymbol{\mu}_{\mathrm{mix}}, \hat{\mathbf{d}}_c \rangle > -D_c/2$. Measuring displacement along direction $\hat{\mathbf{d}}_c$ from $\boldsymbol{\mu}_{\mathrm{mix}}$ towards $\boldsymbol{\mu}_c$, the decision boundary simplifies to $z_c = D_c/2$. Thus, $x$ remains in $\mathcal{C}_{\mathrm{mix}}$ if $z_c < D_c/2$.
\end{proof}

\subsection{Proof of Lemma~\ref{lem:projection_tails} (boundary loss and contamination probabilities)}
\label{app:proof_projection_tails}

\begin{proof}
By high-dimensional projection concentration \cite{klartag2007central,diaconis1984asymptotics}, the 1D marginal projection variable $z_c(X_c^{\mathrm{err}}) = \hat{\mathbf{d}}_c^\top (X_c^{\mathrm{err}} - \boldsymbol{\mu}_{\mathrm{mix}})$ for error samples follows $\mathcal{N}(0, \sigma_{\mathrm{mix},c}^2)$. The upper-tail probability beyond $D_c/2$ is:
\begin{align}
p_c^{\mathrm{loss}} &= \int_{D_c/2}^{\infty} \frac{1}{\sqrt{2\pi}\sigma_{\mathrm{mix},c}} \exp\left( -\frac{u^2}{2\sigma_{\mathrm{mix},c}^2} \right) \mathrm{d}u.
\end{align}
Substituting $t = \frac{u}{\sqrt{2}\sigma_{\mathrm{mix},c}}$:
{\scalefont{0.86}
\begin{align*}
p_c^{\mathrm{loss}} &= \frac{1}{\sqrt{\pi}} \int_{\frac{D_c}{2\sqrt{2}\sigma_{\mathrm{mix},c}}}^{\infty} e^{-t^2} \mathrm{d}t = \frac{1}{2} \left[ 1 - \frac{2}{\sqrt{\pi}} \int_{0}^{\frac{D_c}{2\sqrt{2}\sigma_{\mathrm{mix},c}}} e^{-t^2} \mathrm{d}t \right] \nonumber \\
&= \frac{1}{2}\left[1 - \operatorname{erf}\left( \frac{D_c}{2\sqrt{2}\sigma_{\mathrm{mix},c}} \right)\right].
\end{align*}
}
Symmetrically, for a clean sample $X_c^{\mathrm{clean}} \sim \mathcal{N}(\boldsymbol{\mu}_c, \Sigma_c)$, its relative projection along $\hat{\mathbf{d}}_c$ centered at $\boldsymbol{\mu}_c$ follows $\mathcal{N}(0, \sigma_c^2)$. The lower-tail penetration probability into $\mathcal{C}_{\mathrm{mix}}$ (where $\langle X_c^{\mathrm{clean}} - \boldsymbol{\mu}_c, \hat{\mathbf{d}}_c \rangle < -D_c/2$) yields $q_c^{\mathrm{contam}} = \frac{1}{2}\left[1 - \operatorname{erf}\left( \frac{D_c}{2\sqrt{2}\sigma_c} \right)\right]$.
\end{proof}

\subsection{Proof of Theorem~\ref{thm:global_rate_bounds} (global rate bounds)}
\label{app:proof_global_rate_bounds}

\begin{proof}
We derive the two bounds separately:
\begin{enumerate}[leftmargin=*]
    \item \textbf{Lower Bound for Capture Rate $r_1$}: Total base errors equal $N\varepsilon = \sum_{c=1}^C N \pi_c \varepsilon_c$. For class $c$, the expected number of uncaptured (lost) error samples is $N \pi_c \varepsilon_c p_c^{\mathrm{loss}}$. Summing across $C$ classes gives total lost errors $N_{\mathrm{loss}} = \sum_{c=1}^C N \pi_c \varepsilon_c p_c^{\mathrm{loss}}$. The fraction of captured errors $r_1 = 1 - (N_{\mathrm{loss}} / N\varepsilon)$ satisfies:
    {\small
    \begin{align*}
    r_1  &= 1 - \frac{1}{N\varepsilon} \sum_{c=1}^C N \pi_c \varepsilon_c p_c^{\mathrm{loss}} \\ &= 1 - \frac{1}{\varepsilon} \sum_{c=1}^C \pi_c \varepsilon_c \cdot \frac{1}{2}\left[1 - \operatorname{erf}\left(\frac{D_c}{2\sqrt{2}\,\sigma_{\mathrm{mix},c}}\right)\right].
    \end{align*}
    }
    
    \item \textbf{Upper Bound for Contamination Rate $r_2$}: Total correctly classified samples equal $N(1-\varepsilon)$. For class $c$, clean samples count $N \pi_c (1-\varepsilon_c) \le N \pi_c$. The expected number of clean samples falsely captured into $\mathcal{C}_{\mathrm{mix}}$ is bounded by $N \pi_c q_c^{\mathrm{contam}}$. Summing across all classes and dividing by $N(1-\varepsilon)$:
    {\small
    \begin{align*}
    &r_2 \le \frac{\sum_{c=1}^C N \pi_c q_c^{\mathrm{contam}}}{N(1-\varepsilon)} = \\ &\frac{1}{1-\varepsilon} \sum_{c=1}^C \pi_c \cdot \frac{1}{2}\left[1 - \operatorname{erf}\left(\frac{D_c}{2\sqrt{2}\,\sigma_c}\right)\right].
    \end{align*}
    }
\end{enumerate}
Combining these bounds with the strict monotonicity of $S_{\mathrm{GCUL}}$ (Theorem~\ref{thm:monotonicity}) completes the proof.
\end{proof}

\section{Proofs for sub-Gaussian capture rate analysis}
\label{app:subgaussian_proofs}

\subsection{Proof of lemma for directional chernoff outlier bound (Eq.~\ref{eq:chernoff_outlier})}
\label{app:proof_chernoff}

\begin{proof}
By definition, a sample $X_c \in \mathcal{E}_c$ is a directional outlier escaping $\mathcal{H}_c^+$ if $Z_c \coloneqq \hat{\mathbf{d}}_c^\top (X_c - \boldsymbol{\mu}_{\mathrm{mix}}) > D_c/2$. Applying Chernoff's bounding method, for any $\lambda > 0$:
\begin{align*}
&\mathbb{P}\left( Z_c > \frac{D_c}{2} \right)\\ &= \mathbb{P}\left( e^{\lambda Z_c} > e^{\frac{\lambda D_c}{2}} \right) \le \exp\left(-\frac{\lambda D_c}{2}\right) \mathbb{E}\left[e^{\lambda Z_c}\right].\\
\end{align*}
By the directional sub-Gaussian assumption with variance proxy $\sigma_{\mathrm{mix},c}^2$, we have $\mathbb{E}[e^{\lambda Z_c}] \le \exp\left( \frac{\lambda^2 \sigma_{\mathrm{mix},c}^2}{2} \right)$ (Proposition 2.5.2, \cite{vershynin2018high}. Thus:
\begin{equation}
\mathbb{P}\left( Z_c > \frac{D_c}{2} \right) \le \exp\left( -\frac{\lambda D_c}{2} + \frac{\lambda^2 \sigma_{\mathrm{mix},c}^2}{2} \right).
\end{equation}
Minimizing this quadratic exponent over $\lambda > 0$ yields the optimal parameter $\lambda^* = \frac{D_c}{2\sigma_{\mathrm{mix},c}^2}$. Substituting $\lambda^*$ into the exponent gives $\eta_{\mathrm{outlier}, c} \le \exp\left( -\frac{D_c^2}{8\sigma_{\mathrm{mix},c}^2} \right)$.
\end{proof}

\subsection{Proof of Theorem~\ref{thm:r1_rigorous} (rigorous lower bound for \texorpdfstring{$r_1$}{r1})}
\label{app:proof_r1_rigorous}

\begin{proof}
The total empirical capture rate $r_1$ is given by:
\begin{align*}
r_1 &= \frac{|\mathcal{C}_{\mathrm{mix}} \cap \mathcal{E}|}{|\mathcal{E}|} = \sum_{c=1}^C \frac{|\mathcal{E}_c|}{|\mathcal{E}|} \cdot \frac{|\mathcal{E}_{c, \mathrm{captured}}|}{|\mathcal{E}_c|} \\ &= \sum_{c=1}^C w_c \left(1 - \frac{|\mathcal{E}_{c, \mathrm{loss}}|}{|\mathcal{E}_c|}\right),
\end{align*}
where $w_c = \frac{\pi_c \varepsilon_c}{\varepsilon}$ and $\sum_{c=1}^C w_c = 1$. Since $\mathcal{E}_{c, \mathrm{loss}} \subseteq \mathcal{E}_{c, \mathrm{boundary}} \cup \mathcal{E}_{c, \mathrm{outlier}}$, applying Boole's inequality (union bound) yields:
\begin{align*}
\frac{|\mathcal{E}_{c, \mathrm{loss}}|}{|\mathcal{E}_c|} &\le \frac{|\mathcal{E}_{c, \mathrm{boundary}}|}{|\mathcal{E}_c|} + \frac{|\mathcal{E}_{c, \mathrm{outlier}}|}{|\mathcal{E}_c|}\\ &\le p_c^{\mathrm{loss}} + \eta_{\mathrm{outlier}, c}.
\end{align*}
Substituting this upper bound into $r_1$:
\begin{align*}
r_1 &\ge \sum_{c=1}^C w_c \left( 1 - p_c^{\mathrm{loss}} - \eta_{\mathrm{outlier}, c} \right)\\ &= 1 - \sum_{c=1}^C w_c \left( \eta_{\mathrm{outlier}, c} + p_c^{\mathrm{loss}} \right).
\end{align*}
Substituting $w_c = \frac{\pi_c \varepsilon_c}{\varepsilon}$, $\eta_{\mathrm{outlier}, c} \le \exp\left( -\frac{D_c^2}{8\sigma_{\mathrm{mix},c}^2} \right)$, and $p_c^{\mathrm{loss}} = \frac{1}{2}\left[1 - \operatorname{erf}\left(\frac{D_c}{2\sqrt{2}\sigma_{\mathrm{mix},c}}\right)\right]$ completes the proof.
\end{proof}

\subsection{Proof of corollary~\ref{cor:attractor_properties} (asymptotic convergence)}
\label{app:proof_asymptotic}

\begin{proof}
As $\max_{c} (\sigma_{\mathrm{mix},c}/D_c) \to 0$, $D_c^2 / (8\sigma_{\mathrm{mix},c}^2) \to \infty$. Consequently, $\eta_{\mathrm{outlier}, c} = \exp\left( -\frac{D_c^2}{8\sigma_{\mathrm{mix},c}^2} \right) \to 0$, implying $\mathcal{S}_{\mathrm{exist}} \to 1$. The Chernoff outlier penalty vanishes, and the lower bound in \eqref{eq:r1_rigorous_bound} asymptotically converges to the projection measure concentration bound:
{\scalefont{0.8}
\begin{equation*}
\lim_{\max_c (\sigma_{\mathrm{mix},c}/D_c) \to 0} r_1 \ge 1 - \frac{1}{\varepsilon} \sum_{c=1}^C \pi_c \varepsilon_c \cdot \frac{1}{2}\left[1 - \operatorname{erf}\left(\frac{D_c}{2\sqrt{2}\sigma_{\mathrm{mix},c}}\right)\right].
\end{equation*}
}
\end{proof}

\section{Detailed dataset statistics and empirical profiling}
\label{app:dataset_stats}

To guarantee that GCUL's theoretical guarantees and selective mechanisms remain effective under real-world data heterogeneity, we conduct comprehensive exploratory analysis on three public benchmarks: \textbf{Dair-AI/Emotion}, \textbf{GoEmotions}, and the \textbf{Kaggle Sentiment} dataset. Figures~\ref{fig:appendix_dair_profile}--\ref{fig:appendix_kaggle_profile} visualize their respective class distribution shifts and text length characterizations.

\begin{figure}[H]
    \centering
    \begin{subfigure}[b]{0.36\textwidth}
        \centering
        \includegraphics[width=\textwidth]{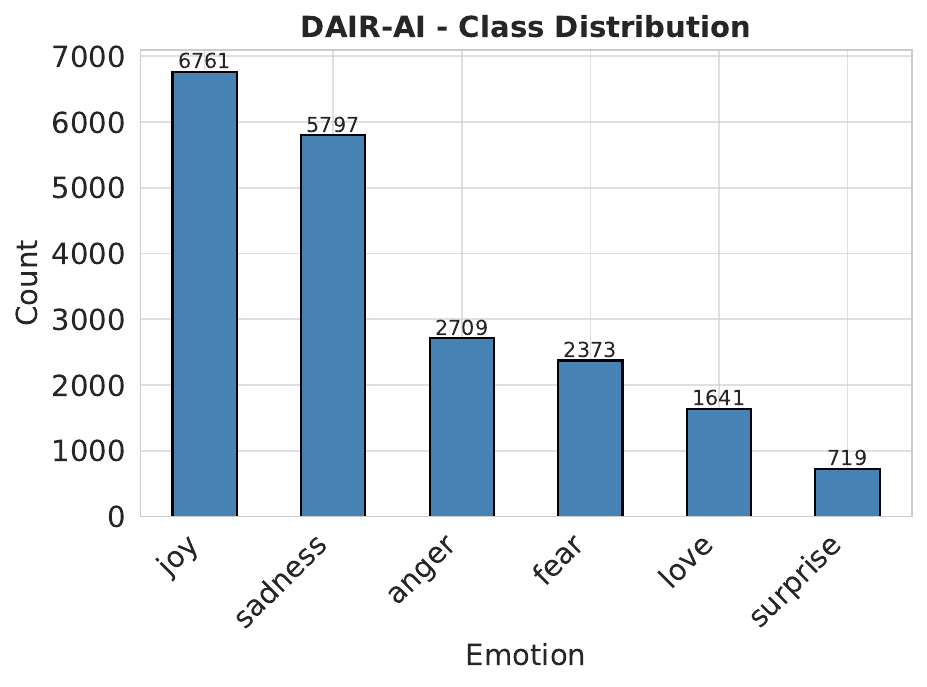}
        \caption{Class count distribution}
        \label{fig:dair_dist}
    \end{subfigure}
    \hfill
    \begin{subfigure}[b]{0.36\textwidth}
        \centering
        \includegraphics[width=\textwidth]{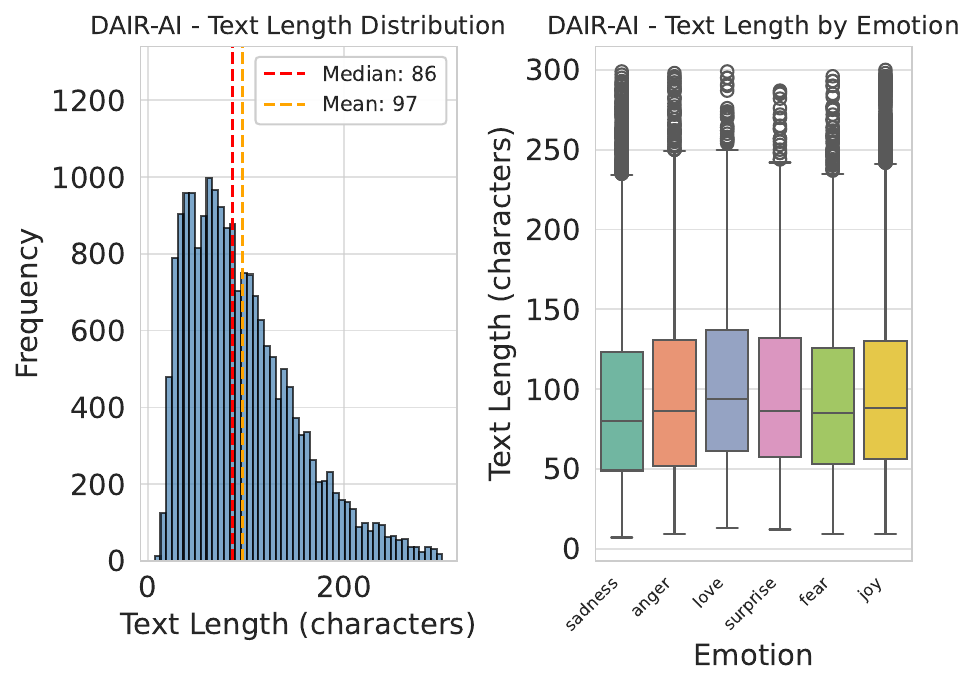}
        \caption{Length distribution \& per-class boxplot}
        \label{fig:dair_len}
    \end{subfigure}
    \caption{\textbf{Dair-AI/emotion dataset profile.} The dataset exhibits a long-tailed class distribution ranging from $6,761$ (\textit{joy}) to $719$ (\textit{surprise}) samples. Sequence lengths concentrate around a median of $86$ characters.}
    \label{fig:appendix_dair_profile}
\end{figure}

\begin{figure}[H]
    \centering
    \begin{subfigure}[b]{0.36\textwidth}
        \centering
        \includegraphics[width=\textwidth]{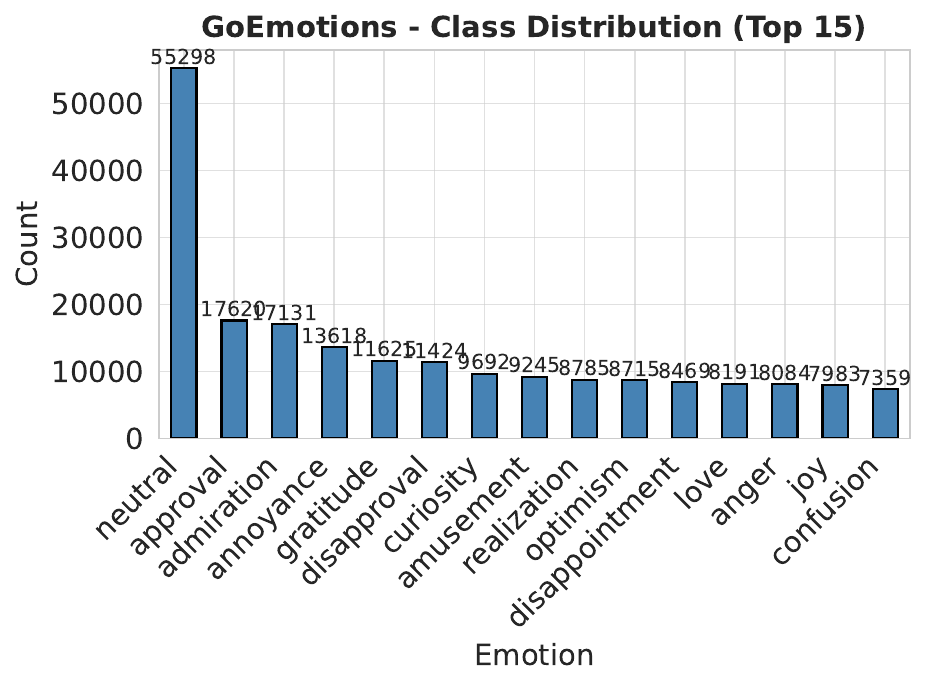}
        \caption{Top-15 class count distribution}
        \label{fig:go_dist}
    \end{subfigure}
    \hfill
    \begin{subfigure}[b]{0.36\textwidth}
        \centering
        \includegraphics[width=\textwidth]{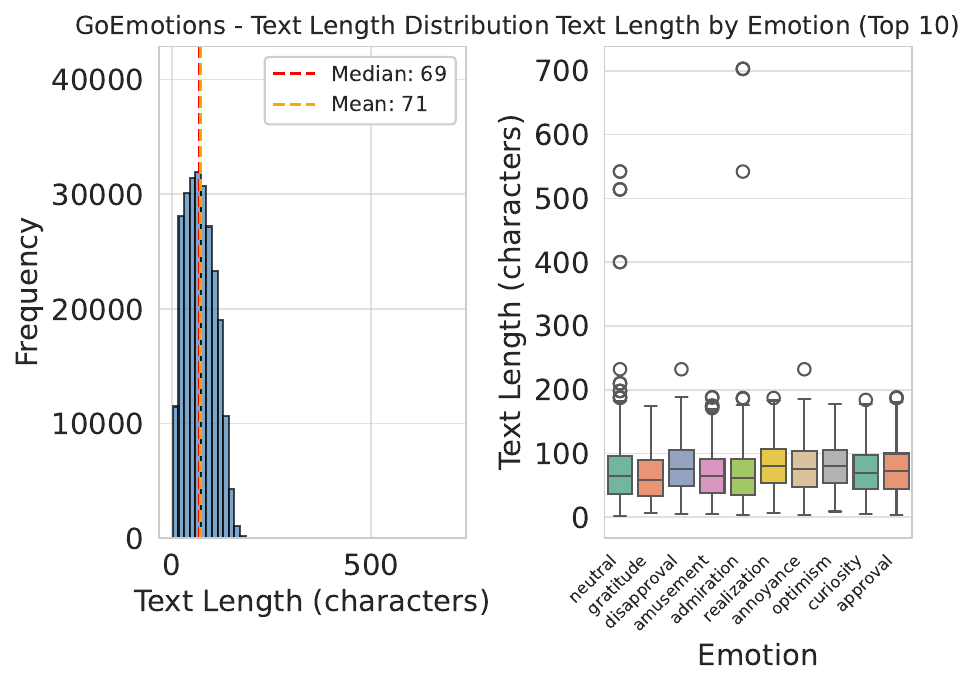}
        \caption{Length distribution \& per-class boxplot}
        \label{fig:go_len}
    \end{subfigure}
    \caption{\textbf{GoEmotions dataset profile.} Highly imbalanced fine-grained taxonomy dominated by $55,298$ \textit{neutral} instances. Texts are concise with a sharp length peak at median $69$ characters.}
    \label{fig:appendix_go_profile}
\end{figure}

\begin{figure}[H]
    \centering
    \begin{subfigure}[b]{0.36\textwidth}
        \centering
        \includegraphics[width=\textwidth]{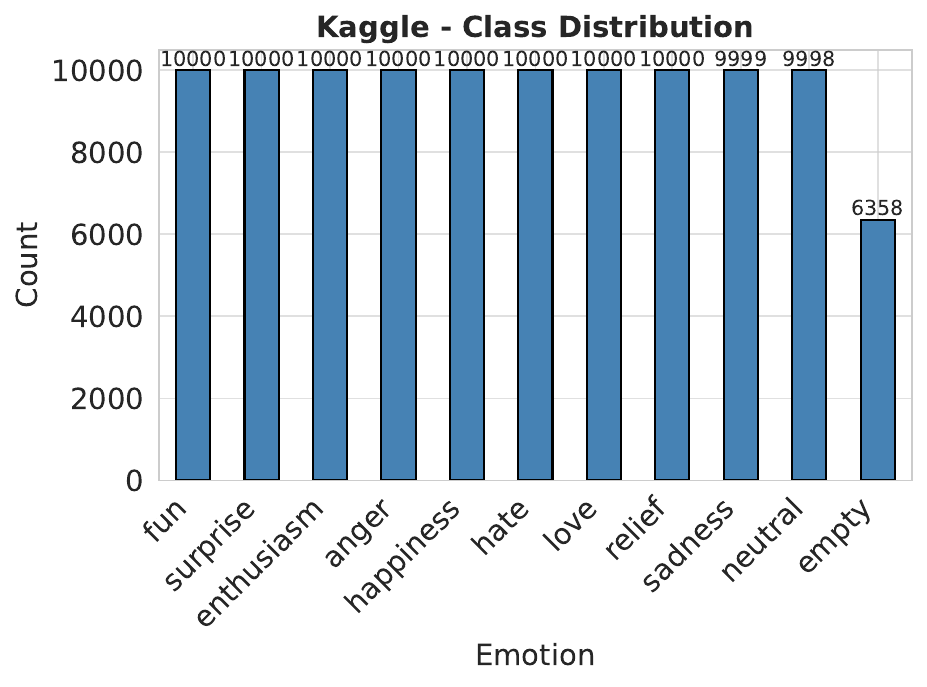}
        \caption{Balanced class distribution}
        \label{fig:kaggle_dist}
    \end{subfigure}
    \hfill
    \begin{subfigure}[b]{0.36\textwidth}
        \centering
        \includegraphics[width=\textwidth]{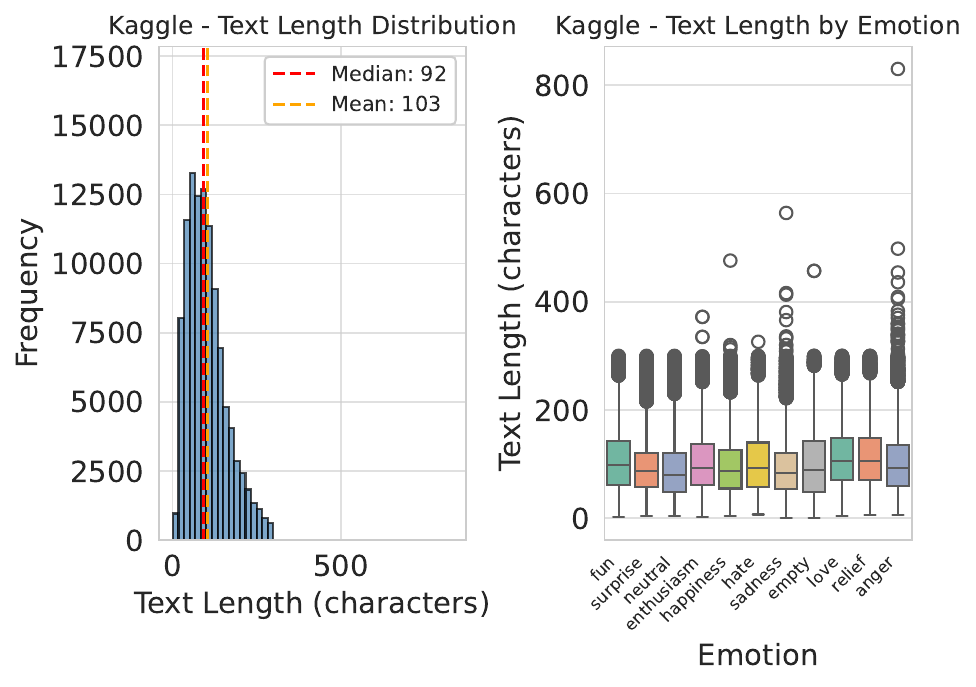}
        \caption{Length distribution \& per-class boxplot}
        \label{fig:kaggle_len}
    \end{subfigure}
    \caption{\textbf{Kaggle emotion dataset profile.} Uniform class balance ($\approx 10,000$ per class) except for the truncated \textit{empty} category ($6,358$). Text lengths feature higher variance (mean $103$ chars) with extended long-tail outliers.}
    \label{fig:appendix_kaggle_profile}
\end{figure}

\paragraph{Key observations on data heterogeneity.}
Across the three evaluated benchmarks, several key structural attributes emerge:
\begin{enumerate}
    \item \textbf{Class Imbalance Spectrum:} The datasets span from strictly uniform distributions (Kaggle: $\approx 10,000$ samples/class) to severe long-tailed distributions (GoEmotions: $55,298$ \textit{neutral} vs. $7,359$ \textit{confusion}). This validates that GCUL's cluster-guided rejection mechanism does not rely on balanced class priors.
    \item \textbf{Sequence Length Invariance:} The median sequence lengths remain relatively short ($69$ to $92$ characters across datasets), but display notable outlier tails up to $800$ characters (e.g., Kaggle \textit{anger}). The per-class boxplots confirm that text length distributions are uniform across semantic emotions, ruling out sequence length as a confounding artifact in cluster formation.
\end{enumerate}

\vspace{-20pt}
\section{Experimental setup and hyperparameter}
\label{app:experimental_details}
\subsection{Benchmark performance on real-world datasets}
\label{app:benchmarks_setup}

\subsubsection{Baseline models implementation}
All models were optimized using early stopping based on validation loss. Feature extraction and hyperparameter choices for baselines are summarized below:

\begin{itemize}
    \item \textbf{TF-IDF Baselines:} Traditional models (Logistic Regression, Linear SVM, Random Forest, Multinomial NB) use TF-IDF features with $N$-gram range $(1, 2)$ and a vocabulary cap of $10,000$. Regularization strength $C=1.0$ for linear models; Random Forest employs $100$ estimators.
    \item \textbf{TextCNN:} Vocab size $25,000$, sequence length $100$, embedded into $300$-d vectors. Filter sizes $\{2, 3, 4\}$ with $100$ feature maps each; dropout rate $0.6$, optimized via Adam ($\text{lr}=5 \times 10^{-4}$).
    \item \textbf{BiLSTM-Attention:} Uses $300$-d pre-trained GloVe embeddings (6B tokens, $84.3\%$ vocabulary coverage, fine-tunable). Architecture includes a $2$-layer BiLSTM ($\text{hidden\_dim}=128$) with self-attention ($\text{dim}=64$), dropout $0.5$, optimized via Adam ($\text{lr}=1 \times 10^{-3}$).
    \item \textbf{DistilBERT:} Fine-tuned using \texttt{distilbert-base-uncased} (66M parameters) with maximum sequence length $128$. Trained for $3$ epochs using AdamW ($\text{lr}=3 \times 10^{-5}$, warmup ratio $0.05$, weight decay $0.01$, batch size $32$).
\end{itemize}

\subsubsection{GCUL framework hyperparameters}

\begin{table}[H]
\centering
\small
\caption{Key hyperparameter configuration for the GCUL framework.}
\label{tab:gcul_hyperparams}
\resizebox{\columnwidth}{!}{
\begin{tabular}{ll}
\toprule
\textbf{Component / Phase} & \textbf{Configuration } \\
\midrule
\textbf{Backbone / Input} & DistilBERT 768-d embeddings \\
\textbf{Sequence Architecture} & 2-layer LSTM \\
\textbf{Optimization} & Adam, Batch size $= 64$ \\
\textbf{Phase Schedule} & Phase 1: 1 epoch; Phase 3: 4 epochs \\
\textbf{Dimensionality Reduction} & 768 $\rightarrow$ 10 dims via PCA / LDA \\
\bottomrule
\end{tabular}
}
\end{table}

\subsection{Comparison with selective classifi-cation baselines}
\begin{table}[H]
\centering
\caption{Experimental parameter configurations for selective classification baselines and GCUL.}
\label{tab:experimental_parameters}
\resizebox{\columnwidth}{!}{
\begin{tabular}{lll}
\toprule
\textbf{Method} & \textbf{Parameter} & \textbf{Configuration} \\
\midrule

\multirow{3}{*}{MSP}
& Rejection cost $\lambda$
& $\{0.1,0.2,\ldots,1.0\}$ \\
& Threshold range
& $\theta \in \{0.00,0.01,\ldots,1.00\}$ \\
& Base classifier
& Logistic Regression, max\_iter $=500$ \\

\midrule

\multirow{7}{*}{SelectiveNet}
& Target coverage
& $\{0.60,0.64,\ldots,0.96\}$ \\
& Utility penalty $\lambda$
& $0.15$ \\
& Training epochs
& $30$ \\
& Learning rate
& $10^{-3}$ \\
& Batch size
& $256$ \\
& Architecture
& LinearSelectiveNet \\
& Loss function
& SelectiveLoss ($\alpha=0.5$, $\lambda=32.0$) \\

\midrule

\multirow{4}{*}{GCUL}
& Rejection cost $\lambda$
& $0.15$ \\
& PCA dimensions
& $\{2,3,4,5,8,10,20,40,80,160,320,768\}$ \\
& Base classifier
& Logistic Regression, max\_iter $=500$ \\
& K-Means
& $n_{\mathrm{init}}=1$, max\_iter $=100$ \\

\midrule

\multirow{2}{*}{Mahalanobis}
& Rejection rate
& $\{0\%,4\%,\ldots,36\%\}$ \\
& Cov regularization
& $\varepsilon=10^{-4}$ \\

\bottomrule
\end{tabular}
} 
\end{table}

\subsection{Synthetic phase-transition verification}

\subsubsection{Dataset generation parameters}

\begin{table}[H]
\centering
\caption{Dataset generation parameters}
\label{tab:synthetic_data_params}
\begin{tabular}{lll}
\toprule
\textbf{Parameter} & \textbf{Value} & \textbf{Description} \\
\midrule
$K$ & 5 & Number of classes \\
$N_{\mathrm{clean}}$ & 2,000 & Samples per clean class \\
$N_{\mathrm{conf}}$ & 2,000 & Samples in confusion region \\
$d$ & 10.0 & Feature dimensionality \\
$S$ & 10.0 & Center scale  \\
$r$ & 1.5 & Cluster radius \\
\bottomrule
\end{tabular}
\end{table}

\subsubsection{GCUL pipeline parameters}

\begin{table}[H]
\centering
\caption{GCUL algorithm parameters}
\label{tab:gcul_params}
\begin{tabular}{lll}
\toprule
\textbf{Parameter} & \textbf{Value} & \textbf{Description} \\
\midrule
$\lambda_{\mathrm{target}}$ & 0.15 & Target threshold \\
\textit{Phase 1} && Logistic Regression \\
\quad solver & lbfgs & Optimization algorithm \\
\quad max\_iter & 500 & Maximum iterations \\
\quad multi\_class & multinomial & Multi-class strategy \\
\textit{Phase 2} && \\
\quad n\_init & 1 & Initialization attempts \\
\quad max\_iter & 100 & Maximum iterations \\
\textit{Phase 3} &&  \\
\quad solver & lbfgs & Optimization algorithm \\
\quad max\_iter & 500 & Maximum iterations \\
\quad multi\_class & multinomial & Multi-class strategy \\
\bottomrule
\end{tabular}
\end{table}

\subsubsection{Noise injection parameters}

\begin{table}[H]
\centering
\caption{Noise sweep parameters}
\label{tab:noise_params}
\begin{tabular}{lll}
\toprule
\textbf{Parameter} & \textbf{Value} & \textbf{Description} \\
\midrule
$\sigma_{\mathrm{noise}}$ & $0, 0.4, 0.8, \dots, 16$ & Feature noise \\
$p_{\mathrm{flip}}$ & $\sigma_{\mathrm{noise}} \times 0.025$ & Label flip probability \\
$\mathcal{S}$ & $\{10, 11, \dots, 19\}$ & Evaluation seeds \\
\bottomrule
\end{tabular}
\end{table}

\subsubsection{Evaluation metrics}

\begin{table}[H]
\centering
\small
\caption{Evaluation metrics}
\label{tab:metrics}
\begin{tabular}{ll}
\toprule
\textbf{Metric} & \textbf{Description} \\
\midrule
Base Acc & Base classifier accuracy on test set \\
Selective Acc & Accuracy after rejection \\
Acc Gain & Selective Acc $-$ Base Acc \\
Uncertain Rate & Proportion of samples rejected \\
PRE $\lambda^*$ & Critical threshold predicted by theory \\
EMP $\lambda^*$ & Critical threshold derived empirically \\
Theory Met & Whether PRE $\lambda^* \geq \lambda_{\mathrm{target}}$ \\
Empirical Effective & Whether EMP $\lambda^* > \lambda_{\mathrm{target}}$ \\
\bottomrule
\end{tabular}
\end{table}

\section{Complete experimental results}
\label{app:extended_results}

\subsection{Comparison with selective classification baselines}
\label{app:comparison_result}

\subsubsection{Kaggle dataset}

\begin{table}[H]
\centering
\small
\caption{MSP results on Kaggle emotion dataset}
\label{tab:kaggle_msp}
\begin{tabular}{cccc}
\toprule
$\lambda$ & Base Acc (\%) & Sel Acc (\%) & Rej Rate (\%) \\
\midrule
0.1 & 88.71 & 99.04 & 16.74  \\
0.2 & 88.71 & 99.04 & 16.74 \\
0.3 & 88.71 & 99.04 & 16.74  \\
0.4 & 88.71 & 99.04 & 16.74 \\
0.5 & 88.71 & 98.12 & 14.56  \\
0.6 & 88.71 & 96.94 & 12.40 \\
0.7 & 88.71 & 91.66 & 4.15 \\
0.8 & 88.71 & 88.72 & 0.01 \\
0.9 & 88.71 & 88.71 & 0.00 \\
1.0 & 88.71 & 88.71 & 0.00 \\
\bottomrule
\end{tabular}
\end{table}

\begin{table}[H]
\centering
\small
\caption{SelectiveNet results on Kaggle emotion dataset}
\label{tab:kaggle_selectivenet}
\begin{tabular}{cccc}
\toprule
Target $c$ & Base Acc (\%) & Sel Acc (\%) & Rej Rate (\%)  \\
\midrule
0.60 & 89.22 & 99.92  & 39.95  \\
0.64 & 89.18 & 99.96 & 35.83  \\
0.68 & 89.33 & 99.92  & 31.95  \\
0.72 & 89.13 & 99.86  & 27.58  \\
0.76 & 89.05 & 99.89  & 23.58  \\
0.80 & 89.17 & 99.68 & 19.83  \\
0.84 & 89.05 & 98.67 & 15.93  \\
0.88 & 89.27 & 96.64 & 11.51  \\
0.92 & 89.24 & 94.30 & 7.88 \\
0.96 & 89.17 & 91.66 & 3.90  \\
\bottomrule
\end{tabular}
\end{table}

\begin{table}[H]
\centering
\small
\caption{Mahalanobis results on Kaggle emotion dataset}
\label{tab:kaggle_mahalanobis}
\begin{tabular}{cccc}
\toprule
Retained & Rej Rate (\%) & Base Acc (\%) & Sel Acc (\%) \\
\midrule
Top 100\% & 0 & 88.71 & 88.71 \\
Top 96\% & 4 & 88.71 & 91.12 \\
Top 92\% & 8 & 88.71 & 93.32  \\
Top 88\% & 12 & 88.71 & 95.65  \\
Top 84\% & 16 & 88.71 & 97.73  \\
Top 80\% & 20 & 88.71 & 99.48  \\
Top 76\% & 24 & 88.71 & 99.93 \\
Top 72\% & 28 & 88.71 & 99.97 \\
Top 68\% & 32 & 88.71 & 99.97 \\
Top 64\% & 36 & 88.71 & 99.97 \\
\bottomrule
\end{tabular}
\end{table}

\begin{table}[H]
\centering
\caption{GCUL results on Kaggle emotion dataset ($\lambda_{\mathrm{target}} = 0.15$)}
\label{tab:kaggle_gcul}
\scalebox{0.76}{
\begin{tabular}{ccccccc}
\toprule
PCA $d$ & Base Acc & Sel Acc & Rej Rate & $\lambda^*_{\mathrm{PRE}}$ & $\lambda^*_{\mathrm{EMP}}$ & Guaranteed \\
\midrule
2 & 88.71 & 92.11 &  5.80 & 0.00 & 0.00 & No \\
3 & 88.71 & 91.95 & 5.64 & 0.03 & 0.05 & No \\
4 & 88.71 & 93.51 &  7.47 & 0.10 & 0.12 & No \\
5 & 88.71 & 93.04 & 6.89 & 0.15 & 0.19 & No \\
8 & 88.71 & 93.93 & 8.07 & 0.32 & 0.38 & Yes \\
10 & 88.71 & 95.47 &  10.22 & 0.35 & 0.46 & Yes \\
20 & 88.71 & 98.19 & 14.44 & 0.29 & 0.48 & Yes \\
40 & 88.71 & 98.17 &  14.49 & 0.29 & 0.48 & Yes \\
80 & 88.71 & 98.19 & 14.50 & 0.30 & 0.49 & Yes \\
160 & 88.71 & 98.18 &  14.48 & 0.31 & 0.50 & Yes \\
320 & 88.71 & 98.18 &  14.48 & 0.32 & 0.52 & Yes \\
768 & 88.71 & 98.18 &  14.48 & 0.34 & 0.59 & Yes \\
\bottomrule
\end{tabular}
}
\end{table}

\subsubsection{GoEmotions dataset}

\begin{table}[H]
\centering
\small
\caption{MSP results on GoEmotions dataset}
\label{tab:goemotions_msp}
\begin{tabular}{cccc}
\toprule
$\lambda$ & Base Acc (\%) & Sel Acc (\%) & Rej Rate (\%) \\
\midrule
0.1 & 37.99 & 100 & 99.83 \\
0.2 & 37.99 & 100 & 99.83 \\
0.3 & 37.99 & 100 & 99.83 \\
0.4 & 37.99 & 100 & 99.83 \\
0.5 & 37.99 & 100 & 99.83 \\
0.6 & 37.99 & 100 & 99.83 \\
0.7 & 37.99 & 37.99 & 0.00 \\
0.8 & 37.99 & 37.99 & 0.00 \\
0.9 & 37.99 & 37.99 & 0.00 \\
1 & 37.99 & 37.99 & 0.00 \\
\bottomrule
\end{tabular}
\end{table}

\begin{table}[H]
\centering
\small
\caption{SelectiveNet results on GoEmotions dataset}
\label{tab:goemotions_selectivenet}
\begin{tabular}{cccc}
\toprule
Target $c$ & Base Acc (\%) & Sel Acc (\%) & Rej Rate (\%) \\
\midrule
0.60 & 37.75 & 45.43 & 39.94 \\
0.64 & 37.47 & 44.30 & 35.94 \\
0.68 & 37.46 & 43.68 & 32.12 \\
0.72 & 37.57 & 42.79 & 28.06 \\
0.76 & 37.94 & 42.45 & 23.82 \\
0.80 & 37.27 & 41.00 & 19.85 \\
0.84 & 37.94 & 41.10 & 16.16 \\
0.88 & 37.86 & 40.07 & 11.85 \\
0.92 & 38.30 & 39.76 & 7.82 \\
0.96 & 38.00 & 38.71 & 3.74 \\
\bottomrule
\end{tabular}
\end{table}

\begin{table}[H]
\centering
\small
\caption{Mahalanobis results on GoEmotions dataset}
\label{tab:goemotions_mahalanobis}
\begin{tabular}{cccc}
\toprule
Retained & Rej Rate (\%) & Base Acc (\%) & Sel Acc (\%) \\
\midrule
Top 100\% & 0 & 37.99 & 37.99 \\
Top 96\% & 4 & 37.99 & 38.14 \\
Top 92\% & 8 & 37.99 & 38.37 \\
Top 88\% & 12 & 37.99 & 38.49 \\
Top 84\% & 16 & 37.99 & 38.69 \\
Top 80\% & 20 & 37.99 & 38.80 \\
Top 76\% & 24 & 37.99 & 38.94 \\
Top 72\% & 28 & 37.99 & 39.05 \\
Top 68\% & 32 & 37.99 & 39.40 \\
Top 64\% & 36 & 37.99 & 39.66 \\
\bottomrule
\end{tabular}
\end{table}

\begin{table}[H]
\centering
\caption{GCUL results on GoEmotions dataset ($\lambda_{\mathrm{target}} = 0.15$)}
\label{tab:goemotions_gcul}
\scalebox{0.76}{
\begin{tabular}{ccccccc}
\toprule
PCA $d$ & Base Acc & Sel Acc & Rej Rate & $\lambda^*_{\mathrm{PRE}}$ & $\lambda^*_{\mathrm{EMP}}$ & Guaranteed \\
\midrule
2 & 37.99 & 38.15 & 3.24 & 0 & 0 & No \\
3 & 37.99 & 38.19 & 2.48 & 0 & 0 & No \\
4 & 37.99 & 38.09 & 3.06 & 0 & 0 & No \\
5 & 37.99 & 38.01 & 2.48 & 0 & 0 & No \\
8 & 37.99 & 38.35 & 4.06 & 0 & 0 & No \\
10 & 37.99 & 38.23 & 2.81 & 0 & 0 & No \\
20 & 37.99 & 38.14 & 1.85 & 0 & 0 & No \\
40 & 37.99 & 37.93 & 2.22 & 0 & 0 & No \\
80 & 37.99 & 38.02 & 2.20 & 0 & 0 & No \\
160 & 37.99 & 37.94 & 2.22 & 0 & 0 & No \\
320 & 37.99 & 37.95 & 2.19 & 0 & 0 & No \\
768 & 37.99 & 38.03 & 2.21 & 0 & 0 & No \\
\bottomrule
\end{tabular}
}
\end{table}

\subsubsection{Dair-AI emotion dataset}

\begin{table}[H]
\centering
\small
\caption{MSP results on Dair-AI emotion dataset}
\label{tab:dair_msp}
\begin{tabular}{cccc}
\toprule
$\lambda$ & Base Acc (\%) & Sel Acc (\%) & Rej Rate (\%) \\
\midrule
0.1 & 92.55 & 98.81 & 11.70 \\
0.2 & 92.55 & 98.81 & 11.70 \\
0.3 & 92.55 & 98.81 & 11.70 \\
0.4 & 92.55 & 97.06 & 8.25 \\
0.5 & 92.55 & 94.09 & 3.55 \\
0.6 & 92.55 & 93.64 & 2.45 \\
0.7 & 92.55 & 92.55 & 0.00 \\
0.8 & 92.55 & 92.55 & 0.00 \\
0.9 & 92.55 & 92.55 & 0.00 \\
1.0 & 92.55 & 92.55 & 0.00 \\
\bottomrule
\end{tabular}
\end{table}

\begin{table}[H]
\centering
\small
\caption{SelectiveNet results on Dair-AI emotion dataset}
\label{tab:dair_selectivenet}
\begin{tabular}{cccc}
\toprule
Target $c$ & Base Acc (\%) & Sel Acc (\%) & Rej Rate (\%) \\
\midrule
0.60 & 93.10 & 98.82 & 40.75 \\
0.64 & 93.00 & 98.43 & 36.50 \\
0.68 & 93.30 & 98.67 & 32.40 \\
0.72 & 93.45 & 98.04 & 28.45 \\
0.76 & 93.10 & 96.74 & 24.75 \\
0.80 & 93.20 & 96.74 & 20.25 \\
0.84 & 93.05 & 96.54 & 16.25 \\
0.88 & 93.15 & 96.57 & 12.60 \\
0.92 & 93.15 & 93.67 & 7.60 \\
0.96 & 93.30 & 93.64 & 4.15 \\
\bottomrule
\end{tabular}
\end{table}

\begin{table}[H]
\centering
\small
\caption{Mahalanobis results on Dair-AI emotion dataset}
\label{tab:dair_mahalanobis}
\begin{tabular}{cccc}
\toprule
Retained & Rej Rate (\%) & Base Acc (\%) & Sel Acc (\%) \\
\midrule
Top 100\% & 0 & 92.55 & 92.55 \\
Top 96\% & 4 & 92.55 & 93.91 \\
Top 92\% & 8 & 92.55 & 93.97 \\
Top 88\% & 12 & 92.55 & 94.20 \\
Top 84\% & 16 & 92.55 & 94.35 \\
Top 80\% & 20 & 92.55 & 94.31 \\
Top 76\% & 24 & 92.55 & 94.28 \\
Top 72\% & 28 & 92.55 & 94.17 \\
Top 68\% & 32 & 92.55 & 93.97 \\
Top 64\% & 36 & 92.55 & 93.75 \\
\bottomrule
\end{tabular}
\end{table}

\begin{table}[H]
\centering
\caption{GCUL results on Dair-AI emotion dataset ($\lambda_{\mathrm{target}} = 0.15$)}
\label{tab:dair_gcul}
\scalebox{0.76}{
\begin{tabular}{ccccccc}
\toprule
PCA $d$ & Base Acc & Sel Acc & Rej Rate & $\lambda^*_{\mathrm{PRE}}$ & $\lambda^*_{\mathrm{EMP}}$ & Guaranteed \\
\midrule
2 & 92.6 & 93.65 & 2.35 & 0.00 & 0.00 & No \\
3 & 92.6 & 95.92 & 5.70 & 0.00 & 0.25 & No \\
4 & 92.6 & 95.79 & 4.95 & 0.22 & 0.27 & Yes \\
5 & 92.6 & 95.73 & 5.10 & 0.33 & 0.31 & Yes \\
8 & 92.6 & 95.79 & 4.95 & 0.10 & 0.24 & No \\
10 & 92.6 & 95.69 & 4.95 & 0.11 & 0.23 & No \\
20 & 92.6 & 95.79 & 4.95 & 0.18 & 0.33 & Yes \\
40 & 92.6 & 95.69 & 4.95 & 0.21 & 0.38 & Yes \\
80 & 92.6 & 95.69 & 4.95 & 0.24 & 0.42 & Yes \\
160 & 92.6 & 95.69 & 4.95 & 0.25 & 0.48 & Yes \\
320 & 92.6 & 95.69 & 4.95 & 0.28 & 0.56 & Yes \\
768 & 92.6 & 95.69 & 4.95 & 0.32 & 0.70 & Yes \\
\bottomrule
\end{tabular}
}
\end{table}

\subsection{synthetic data results (seed 16)}
\label{app:synthetic_result}
\vspace{-10pt}

\begin{table}[H]
\centering
\small
\caption{GCUL performance under varying noise levels on synthetic data (seed 16)}
\label{tab:synthetic_gcul}
\scalebox{0.9}{
\begin{tabular}{ccccccc}
\toprule
Noise  & Base Acc & Sel Acc  & Rej Rate & $\lambda^*_{\mathrm{PRE}}$ & $\lambda^*_{\mathrm{EMP}}$ \\
$\sigma_f$ & (\%) & (\%) & (\%) & & \\
\midrule
0.0   & 86.50 & 100  & 16.54 & 0.96 & 0.82 \\
0.4  & 86.54 & 100  & 16.54 & 0.96 & 0.81 \\
0.8  & 86.88 & 100  & 16.54 & 0.95 & 0.79 \\
1.2   & 86.92 & 100  & 16.54 & 0.94 & 0.79 \\
1.6   & 86.92 & 100  & 16.54 & 0.92 & 0.79 \\
2.0   & 86.92 & 100  & 16.67 & 0.90 & 0.79 \\
2.4   & 87.13 & 99.95  & 16.71 & 0.87 & 0.77 \\
2.8   & 87.08 & 99.25  & 16.21 & 0.82 & 0.75 \\
3.2   & 86.96 & 98.36  & 15.96 & 0.75 & 0.72 \\
3.6   & 86.79 & 97.37  & 16.17 & 0.68 & 0.65 \\
4.0   & 86.13 & 96.38 & 16.08 & 0.60 & 0.64 \\
4.4   & 85.38 & 94.95  & 15.88 & 0.53 & 0.60 \\
4.8  & 84.21 & 92.53  & 15.25 & 0.46 & 0.55 \\
5.2   & 82.96 & 90.08 & 14.33 & 0.39 & 0.50 \\
5.6   & 81.46 & 86.79  & 13.25 & 0.33 & 0.40 \\
6.0   & 80.08 & 83.73  & 13.21 & 0.27 & 0.28 \\
6.4  &  78.71 & 81.36  & 13.29 & 0.22 & 0.21 \\
6.8  &  76.83 & 79.24  & 13.08 & 0.18 & 0.18 \\
7.2  &  75.00 & 77.26  & 12.42 & 0.15 & 0.18 \\
7.6  &  73.46 & 75.42  & 13.38 & 0.11 & 0.15 \\
8.0  &  71.33 & 73.31  & 13.67 & 0.07 & 0.15 \\
8.4  & 69.54 & 71.10  & 13.79 & 0.03 & 0.11 \\
8.8  & 68.00 & 68.85 & 14.00 & 0.00 & 0.06 \\
9.2  & 65.71 & 66.91 & 14.88 & 0.00 & 0.08 \\
9.6  & 63.96 & 65.83 & 13.54 & 0.00 & 0.14 \\
10.0   & 62.38 & 64.46  & 13.96 & 0.00 & 0.15 \\
10.4   & 60.71 & 62.64  & 13.67 & 0.00 & 0.14 \\
10.8   & 59.29 & 60.66  & 14.42 & 0.00 & 0.10 \\
11.2   & 57.92 & 59.88  & 15.46 & 0.00 & 0.13 \\
11.6   & 56.79 & 58.21  & 14.96 & 0.00 & 0.10 \\
12.0  & 55.75 & 57.25  & 15.21 & 0.00 & 0.10 \\
12.4   & 54.67 & 55.21 & 15.17 & 0.00 & 0.04 \\
12.8   & 53.21 & 53.81  & 15.29 & 0.00 & 0.04 \\
13.2   & 52.13 & 52.42  & 15.58 & 0.00 & 0.02 \\
13.6   & 51.21 & 51.88 & 16.00 & 0.00 & 0.04 \\
14.0   & 50.04 & 50.82  & 16.46 & 0.00 & 0.05 \\
14.4   & 49.08 & 49.60  & 16.50 & 0.00 & 0.03 \\
14.8  & 48.25 & 48.27  & 16.71 & 0.00 & 0.00 \\
15.2   & 47.50 & 47.39  & 17.00 & 0.00 & -0.01 \\
15.6   & 46.54 & 46.80  & 17.21 & 0.00 & 0.02 \\
16.0   & 45.71 & 45.61 & 16.96 & 0.00 & -0.01 \\
\bottomrule
\end{tabular}
}
\end{table}

\section{Code availability}
The complete source code is available at: \url{https://github.com/u678868/GCUL}

\end{document}